\documentclass{article}

\usepackage{mathtools}
\usepackage{amssymb,latexsym,amsmath, mathtools,bm,amsthm}
\usepackage{color}
\usepackage[dvipsnames]{xcolor}
\usepackage{amssymb}
\usepackage{bbm,mathrsfs}
\usepackage{multicol,multirow}
\usepackage{setspace}
\newtheorem{theorem}{Theorem}
\newtheorem{lemma}{Lemma}
\newtheorem{definition}{Definition}

\newtheorem{corollary}{Corollary}

\def\mb{\mathbb}
\def\mbf{\mathbf}
\def\mc{\mathcal}
\def\mk{\mathfrak}

\newcommand{\E}{\mb{E}}
\renewcommand{\P}{\mb{P}}

\renewcommand{\v}[1]{\mbf{#1}}

\newcommand\redout{\bgroup\markoverwith{\textcolor{red}{\rule[.5ex]{2pt}{0.4pt}}}\ULon}

\usepackage{pifont}
\usepackage[preprint, nonatbib]{neurips_2022}

\usepackage[utf8]{inputenc} 
\usepackage[T1]{fontenc}    

\usepackage{booktabs,hyperref,url}       
\usepackage{amsfonts}       
\usepackage{nicefrac}       
\usepackage{microtype}      
\usepackage{xcolor}         

\title{Differentially Private Learning of Hawkes Processes}

\makeatletter
\newcommand{\printfnsymbol}[1]{%
  \textsuperscript{\@fnsymbol{#1}}%
}
\makeatother

\author{%
  Mohsen Ghassemi\thanks{Equal contribution} \and 
    \textbf{Eleonora Krea\v{c}i\'{c}}\printfnsymbol{1} \and 
    \textbf{Niccol\`o Dalmasso} \AND
    \textbf{Vamsi K. Potluru} \and
    \textbf{Tucker Balch} \and
    \textbf{Manuela Veloso} \AND
{\rm J.P. Morgan AI Research}\\[2pt]
$\{$\texttt{mohsen.ghassemi}, \texttt{eleonora.kreacic}, \texttt{niccolo.dalmasso}$\}$\texttt{@jpmchase.com}
}

\begin{document}

\maketitle

\begin{abstract}
  Hawkes processes have recently gained increasing attention from the machine learning community for their versatility in modeling event sequence data. While they have a rich history going back decades, some of their properties, such as sample complexity for learning the parameters and releasing differentially private versions, are yet to be thoroughly analyzed.
  In this work, we study standard Hawkes processes with background intensity $\mu$ and excitation function $\alpha e^{-\beta t}$. 
  We provide both non-private and differentially private estimators of $\mu$ and $\alpha$, and obtain sample complexity results in both settings to quantify the cost of privacy.
  Our analysis exploits the strong mixing property of Hawkes processes and classical central limit theorem results for weakly dependent random variables. We validate our theoretical findings on both synthetic and real datasets.
\end{abstract}

\section{Introduction}

Event sequence data capture the unfolding of a discrete set of events over time. Events data are ubiquitous in many domains, with frequencies spanning many time scales. In finance, limit order book data record the orders placed on specific securities every tenth of a second \cite{bacry2015hawkes}. In advertising, daily clickstream data detail the search habits of consumers \cite{laxman2008stream} or the popularity and interactions within a social media network \cite{farajtabar2015coevolve}. In geology, earthquake logs report the strength of surface movements over several decades \cite{ogata1998space}. Unlike time series however, events data are asynchronous \cite{ross1996stochastic}. In other words, the time between events is not regular and is in fact a fundamental piece in understanding the dynamics of the event data. Temporal point processes \cite{cox1980point} are a powerful mathematical tool to model inter-arrival times between events. One can use point processes not only to learn event dynamics but also to simulate event sequences with the same inter-arrival dynamic \cite{ogata1981lewis}. Hawkes processes \cite{hawkes1971spectra, isham1979self} are a class of point processes that capture self-excitation dynamics, i.e., situations in which the arrival of one event increases the rate of the next. Hawkes processes have recently gained increasing attention from the machine learning community due to their versatility in modeling event sequence data. Successful applications include topic modeling and clustering for text document streams \cite{he2015hawkestopic, du2015dirichlet}, network structure reconstruction and inference \cite{yang2013mixture, choi2015constructing, etesami2016learning}, personalized recommendations \cite{du2015time} and learning causality structures \cite{xu2016learning}, among others. Recent works have focused on deep learning approaches to estimate the Hawkes process intensity function \cite{mei2017neural, zuo2020transformer, zhang2020self}. 

In certain applications, however, the use of Hawkes processes for both inference and simulation with event sequence data can raise privacy concerns. In epidemiology, infection events data are used to model the infection dynamics and eventually inform mitigation measures, but the personal information of specific infected individuals should remain anonymous. In finance, modeling trading activity on financial securities is key for market makers, but one must be careful not to leak information about a specific client activity. Differential privacy (DP) \cite{dwork2006calibrating} provides theoretical guarantees that, from information that is available, potential adversary is not able to distinguish whether a particular individual was present in a dataset or not. 
The standard DP setting assumes independence of the records in a database. However, 
event sequences generated by Hawkes processes come with the caveat that a single individual is identifiable not only through their own activity, but all the activities they have influenced or by which they have been influenced. 
Thus a potential adversary would need to be prevented from identifying \textit{all activities linked to a single individual}. 

Overall, while Hawkes processes have a rich history going back decades, some of their properties, such as sample complexity for learning the parameters and releasing differentially private versions, are yet to be thoroughly analyzed. In this work we provide results on sample complexity of the estimates on Hawkes process. To the best of our knowledge, this is the first result of such kind. We also introduce private versions of the estimates and study their sample complexity.\\
In summary, our contribution is threefold:
\begin{enumerate}
    \item We obtain first sample complexity result for the estimates of Hawkes process (Section~\ref{sect_estimating_params}). Sample complexity is in terms of an upper bound on the minimum period of time Hawkes process needs to be observed for the estimator to reach within desired distance of true parameters. 
    \item We provide $(\gamma,\epsilon)$-random differentialy private versions of the estimates, and obtain sample complexity results in this context (Sections~\ref{sect_DP_modeling} and \ref{sect_estimating_params_private}).
    \item Based on the above two, we obtain theoretical result on the \textit{cost of privacy}, measured by additional time required for private estimates to reach within desired distance of true ones, in comparison to their non-private counterparts. We also access privacy-utility tradeoff empirically through a series of numerical experiments (Section~\ref{sect_experiments}).
\end{enumerate}


We refer the reader to the supplementary materials for proofs and further experiments.

\subsection{Related work}\label{sect_related_work}

Differential privacy was proposed by Dwork et al. \cite{dwork2006calibrating} and has since been employed in a wide range of applications 
\cite{dwork2014algorithmic,price2019privacy,tambe2019artificial,hassan2019differential,tang2017privacy,zhu2017differentially,mendes2017privacy,abadi2016deep,mcmahan2017learning,ding2017collecting,cho2020contact}. In this work, we consider the definition of DP presented in the original work of Dwork et al. \cite{dwork2006calibrating} as well as the \textit{random differential privacy} setting proposed by Hall et al. \cite{hall2013random}. In the random differentially private setting, the data samples are seen as random draws from an unknown distribution and differential privacy is guaranteed with high probability with respect to the realizations of the underlying data generating distribution. Note that random differential privacy is different from approximate differential privacy (also known as $(\epsilon, \delta)$-differential privacy) \cite{dwork2014algorithmic} as well as probabilistic differential privacy \cite{lin2021privacy} in which the high probability argument is with respect to the randomness in the outcome of the randomized mechanism, 
and not over the realizations of the unknown data generating distribution.

Our work is closely related to the growing literature on differential privacy for non-i.i.d. data.  In a non-i.i.d. setting, presence of each data point may reveal information about presence of data points correlated to it in the dataset. Hence, removing one entry from the dataset may impact other entries as well. When analyzing privacy guarantees in this scenario it is thus not sufficient to simply remove one entry from the dataset, but any ``trace'' of the presence of such entry needs to be removed. 
Differential privacy in non-i.i.d. settings has been analyzed for time series data \cite{cao2018quantifying} and tabular data \cite{lv2018correlated,zhao2017dependent,zhang2020correlated,zhu2014correlated,song2017pufferfish}. In comparison, in this work we analyze privacy for modeling sequential events data, another setting where the records (i.e. occurrence of the events) are often correlated with one another. 


Our work on sample complexity and the privacy analysis relies on the line of work on statistical properties of Hawkes processes. In our analysis we take advantage of stationarity of Hawkes processes \cite{laub2015hawkes,gao2018functional}, and strong mixing property of the count series associated with a Hawkes process \cite{cheysson2021strong} which allows us to use analytical tools used for weekly dependent data \cite{tikhomirov1981convergence, bentkus1997berry}. Our analysis further relies on the moments of Hawkes process which have been studied by Cui et al. \cite{cui2020elementary} and Daw and Pender \cite{daw2018ephemerally, daw2018queues}, among others. Moreover, in order to analyze the sensitivity of the differentially private estimator, we turn to the immigration-birth representation of Hawkes processes and view Hawkes processes as branching processes \cite{athreya2004branching}. This allows us to build upon results obtained by Daw and Pender \cite{daw2018ephemerally} to find probabilistic upper bounds on the number of records (events) that are correlated. We note that our results are different in nature from the line of work on asymptotic behavior of Hawkes processes \cite{torrisi2022asymptotic, bordenave2007large, hillairet2021malliavin,cheysson2021strong}, in that we focus on obtaining error bounds and sample complexity results for estimation of Hawkes parameters. 
Our theoretical estimation bounds (in the non-private setting) are most closely related to the regret bounds obtained by \cite{yang2017online} for an online maximum likelihood estimator of Hawkes processes. Our results, in contrast, give bounds directly on the estimation error of the Hawkes parameters.   


In sample complexity literature, in general there are two types of results, one involves learning the densities and other is parameter learning. In this paper, we are focused on the latter namely parameter learning similar to ~\cite{krishnamurthy2020algebraic} and hence do not need to concern ourselves with proper vs improper learning. Note that while the homogenous case of Poisson process reduces to learning the parameter of an exponential distribution, it immediately jumps in difficulty the moment we add self-excitation. We can find a good overview of sample complexity in learning distributions in ~\cite{moitra2018algorithmic,diakonikolas2016learning}.


\section{Preliminaries}\label{sect_prelim}
\paragraph{Hawkes Processes}
In this paper, we study Hawkes process with intensity given by 
\begin{align}\label{hawkes_def}
    \lambda^*_{t} = \mu + \int_{-\infty}^{t} \alpha e^{-\beta(t-s)}dN(s) = \mu+\sum_{t_{i}<t}\alpha e^{-\beta(t_{i}-t)},
\end{align}
where $\mu$ corresponds to the exogenous base rate of Hawkes process, i.e. the rate of the background Poisson process and $\sum_{t_{i}<t}\alpha e^{-\beta(t_{i}-t)}$ captures the indigenous component, i.e. the impact of previous events. The exponential impact function $\alpha e^{-\beta(t_{i}-t)}$ is parameterized by the excitation rate $\alpha$ and the decay rate $\beta$. In this work, we study regime $\alpha<\beta$ in a simplified variant of Hawkes parameter estimation where the decay rate $\beta$ is known to the learner. For simplicity, in this work we assume $\beta=1$. For the  Hawkes process $\mbf H(t)$ defined by \ref{hawkes_def}, associated count process $\mbf N(t) = (N(t): t\geq 0)$ is given by $N(t) = H(t)-H(0)$.

\paragraph{Differential privacy.} Differential privacy is a quantitative definition of privacy that measures the privacy leakage from publishing functions of private data. 
\begin{definition}\cite{dwork2014algorithmic}
A randomized mechanism $\mk M: \mc X^n \rightarrow \mc Y$ is $(\epsilon,\delta)$-differentially private if for any two datasets $\mc D, \mc D'\in \mc X^n$ that differ in only one entry, 
we have
\[
\forall \mc C\subseteq \mc Y,~\quad \P(\mk M(\mc D) \in \mc C) \leq e^{\epsilon} ~\P(\mk M(\mc D') \in \mc C) +\delta.
\]
When $\delta =0$, we say $\mk M$ satisfies $\epsilon$-differential privacy.
\end{definition}
Many differentially private mechanisms are built on the idea of adding deliberate noise to a non-private mechanism. In this work we employ one such mechanism, the Laplace mechanism, which is defined as
\[
M_{Lap}(\mc D, f(.), \epsilon) = f(\mc D) + \Lambda(0, \Delta_f/\epsilon),
\]
where $\Delta_f = \max_{\mc D,\mc D'}  \|f(\mc D) - f( \mc D' )\|_1$  is the $\ell_1$ sensitivity of $f$ with respect to change of a single query in the dataset, and $\Lambda$ denotes a Laplace random variable parametrized by the mean and scale. Note that here $\mc D,\mc D'\in \mc X^n$ are two neighboring (adjacent) datasets, meaning that they differ in only one entry.

Since the introduction of $(\epsilon,\delta)$-differential privacy, other closely notions of differential privacy have been introduced \cite{mironov2017renyi,dwork2016concentrated,hall2013random}. This work makes use of a relaxation of the standard $\epsilon$-differential privacy definition called \textit{random differential privacy}.
\begin{definition}\label{def_random_dp}
A randomized mechanism $\mk{M}: \mc X^n\to\mc Y$ is $(\epsilon,\gamma)$-randomly differentially private if
\begin{align}
    \P\Big( \forall \mc C\subset\mc Y, \quad \mathbb{P}\left(\mk M(\mc D)\in \mc C \right)\leq e^{\epsilon} \P\left(\mk M(\mc D')\in \mc C \right) \Big)\geq 1-\gamma
\end{align}
where the inner probability is over the randomness of the mechanism, and the outer probability is over neighbouring datasets $\mc D, \mc D'\in \mc X^n$ drawn from distribution $P$ on the space $\mc X^n$.
\end{definition}



\section{Estimating parameters of Hawkes process}\label{sect_estimating_params}


In order to estimate parameters $\mu$ and $\alpha$ of a Hawkes process $\mbf H(t)$ defined by \ref{hawkes_def}, we study the discrete time count data time series (counts series for short) associated with $\mbf H(t)$, generated by the count measure on consecutive intervals of size $\Delta$. We denote the count series of $\mbf H(t)$ with interval size $\Delta$  by 
$\{Y_{i}(\Delta) = \mbf N(i\Delta)- \mbf N((i-1)\Delta) \}_{i=1}^K$ where $\mbf N(t)$ is the count process associated with $\mbf H(t)$, $K=T/\Delta$, and $T$ represents the length of the observation period.

Let $\lambda_{\infty} \coloneqq \lim_{t\to \infty} \E[\lambda^*_t]=\frac{\mu}{1-\alpha}$. Assuming that the Hawkes process under consideration has started at $-\infty$, then an any time during the observation period $t\in [0, T]$, we have $\E[\lambda^*_t] = \lambda_{\infty}$. Under this standard \textit{stationarity} assumption (see e.g. Section 2.1 of \cite{gao2018functional} ), 
for the count series $Y_1({\Delta}), Y_2({\Delta}),\cdots, Y_{K}({\Delta})$ we have \cite{daw2018queues}:
\begin{align} 
    &\eta \coloneqq\E[Y_i(\Delta)]= \frac{\mu \Delta}{(1-\alpha)} \label{eq_stationary_mean_hawkes}\\ 
    &\sigma^2 \coloneqq\text{Var}[Y_i(\Delta)] = \frac{\mu\Delta}{(1-\alpha)^3} +\frac{\alpha^2 \mu \left(1-e^{-2(1-\alpha)\Delta}\right)}{2(1-\alpha)^4}-\frac{2\alpha\mu \left(1-e^{-(1-\alpha)\Delta}\right)}{(1-\alpha)^4} \label{eq_stationary_var_hawkes}
\end{align}
for all $i\in {1, 2,\cdots, K}$. The detailed derivation of equations \eqref{eq_stationary_mean_hawkes} and \eqref{eq_stationary_var_hawkes} is provided in Supplementary Material, Section~\ref{appendix_calculation_mean_var_Hawkes_stationary}.



%
%
%
From equations \eqref{eq_stationary_mean_hawkes} and \eqref{eq_stationary_var_hawkes} one can compute $\mu$ and $\alpha$ given the values of $\mu$ and $\sigma^2$. Therefore,
by deriving the \emph{sample mean $\hat \eta$} and the \emph{sample variance} $\hat \sigma^2$ of $Y_i$, one can obtain estimates for $\hat{\mu}$ and $\hat{\alpha}$. 
%
%
Note that the random variables $Y_i(\Delta)$ are not independent. However, we use the strong mixing property of stationary Hawkes processes and their associated count series \cite{cheysson2021strong} which enables us to employ standard results on statistical properties of weakly dependent random variables \cite{tikhomirov1981convergence}, and thus obtain sample complexity results. We present our sample complexity results in Section~\ref{sec_sample_complexity_nonprivate} below.

\subsection{Sample complexity of the estimator} \label{sec_sample_complexity_nonprivate}

In this section we obtain an upper bound on the minimum period of time the Hawkes process needs to be observed (i.e. the minimum length of sequence) required for the estimator to reach within a desired distance of the true parameters of the Hawkes process. Here, we abuse the word ``sample'' and call this requirement the \textit{sample complexity}. To the best of our knowledge our analysis is the first to study sample complexity for learning of Hawkes processes.

Let $\eta_{4}(\Delta) = \E[(N_\Delta)^4]$ denote the fourth moment of $N_\Delta$ in the stationary regime (see e.g. \cite{cui2020elementary} or \cite{daw2018ephemerally}).  Let $\Phi$ denote CDF of standard normal distribution.
For a fixed bins size $\Delta$, our sample complexity result determines minimal observation period $T=K\Delta$ that is needed in order to achieve desired precision with desired probability.

\begin{theorem}\label{thm_non_private_params}
Let $\hat\mu$ and $\hat \alpha$ be the estimates of parameters $\mu_{\text{lower}}<\mu<\mu_{\text{upper}}$ and $\alpha^{\text{lower}}<\alpha<\alpha^{\text{upper}}$ of Hawkes process started from stationarity. Let $\Psi(\cdot)$ denote the inverse CDF of the standard normal distribution. Let $T=K\Delta$ denote the length of time the Hawkes process is observed, divided into $K$ intervals of size $\Delta$ such that $\Delta > \frac{4C_{9}\mu^{\text{upper}}}{(1-\alpha^{\text{upper}})^4\xi}$ for some $0<\xi<\frac{C_{9}\mu_{\text{lower}}}{6}$ where $C_{9}=\max\{ \frac{8}{\mu_{\text{lower}}(1-\alpha_{\text{upper}})}, 1+\frac{8\mu_{\text{upper}}}{\mu_{\text{lower}}(1-\alpha_{\text{upper}})^2}+\frac{4}{3(1-\alpha_{\text{upper}})} \}$.
Then, if 
\begin{align}\label{condition_T_nonprivate}
T\geq \frac{\sigma^2}{\xi} \max\Big\{ \frac{C_{9}^{2}\Psi(1-\frac{\delta}{8})^{2}}{\xi\Delta},
\frac{9C_{9}^{2}\Psi(1- \frac{\delta}{16})^2(\eta_4 - \sigma^2)}{\xi\Delta},
 3C_{9}\Psi( 1-\frac{\delta}{16}) ^2, 
 \frac{24C_{9}}{\delta}\Big\},
\end{align}
for some $0<\delta\leq 1$, we have $\mathbb{P}\left( \left\vert \alpha - \hat{\alpha}\right\vert >\xi\right)\leq \delta$ and $\mathbb{P}\left(\left\vert \mu-\hat{\mu} \right\vert>\xi\right)\leq \delta$.
%


\end{theorem}
\begin{proof}[Proof sketch]

In order to prove $\mathbb{P}\left( \left\vert \alpha - \hat{\alpha}\right\vert >\xi\right)\leq \delta$ and $\mathbb{P}\left(\left\vert \mu-\hat{\mu} \right\vert>\xi\right)\leq \delta$, we show that it suffices to have $\mathbb{P} \big(| \hat\eta - \eta | >\frac{\xi \Delta}{C_{9}}\big)<\delta/2$ and $\mathbb{P} \big(| \hat\sigma^2 - \sigma^2 | >\frac{\xi \Delta}{C_{9}}\big)<\delta/2$. 
The strong mixing property of Hawkes processes allows us to employ Berry-Essen theorem for weakly dependent random variables \cite{tikhomirov1981convergence} to show that the condition $\mathbb{P} \big(| \hat\eta - \eta | >\frac{\xi \Delta}{C_{9}}\big)<\delta/2$ on the sample mean is satisfied by the first term in condition \eqref{condition_T_nonprivate}. We now turn to sample variance. We have $\hat{\sigma}^2 - \sigma^2 =
\frac{1}{K}\sum_{i=1}^{K}[(Y_{i}-\eta)^2-\sigma^2] - ( \frac{1}{K}\sum_{i=1}^{K}Y_{i}-\eta )^2+\frac{\hat{\sigma}^2}{{K}}$. Therefore, to guarantee $\mathbb{P} \big(| \hat\sigma^2 - \sigma^2 | >\frac{\xi \Delta}{C_{9}}\big)<\delta/2$, it is sufficient to show 
\begin{align}
    &\mathbb{P} \big(\big|\frac{1}{K} \sum\nolimits_{i=1}^{K}(Y_{i}-\eta)^2-\sigma^2\big| >\frac{\xi \Delta}{3C_{9}}\big)<\frac{3\delta}{16},\label{condition_nonprive_2}\\
    &\mathbb{P} \big(\big|\frac{1}{K} \sum\nolimits_{i=1}^{K}Y_{i}-\eta \big|^2 >\frac{\xi \Delta}{3C_{9}}\big)<\frac{3\delta}{16},\label{condition_nonprive_3} \\
    &\mathbb{P} \big(\frac{\hat\sigma^2}{K} >\frac{\xi \Delta}{3C_{9}}\big)<\frac{\delta}{8}.\label{condition_nonprive_4}
\end{align}
By employing Berry-Essen theorem for weakly dependent RVs as well as Markov's inequality, we show that conditions \eqref{condition_nonprive_2}, \eqref{condition_nonprive_3}, and \eqref{condition_nonprive_4} are respectively satisfied by the second, the third, and the fourth term in condition \eqref{condition_T_nonprivate}. 
The detailed proof of Theorem \ref{thm_non_private_params} is provided in section \ref{appendix_non_private_estimators} of the supplementary materials.
\end{proof}

\section{DP modeling of sequential events data}\label{sect_DP_modeling}

We now discuss differentially private modeling of sequential events data. Consider a sequence of events $S(t) = \{e_i = (x_i,t_i)| t_i<t\}$. We assume the occurrence of the events is governed by a Hawkes process defined in \eqref{hawkes_def}. In this paper, we study the setting where $t_i$ denote the timestamp of the $i$-th event and $x_i$ contains the identity of the individual associated with that event. 
We aim to learn the true parameters of the underlying Hawkes process while preserving the privacy of the individuals whose data (events) are present in the stream. That is, we want to make it impossible for an adversary to infer with certainty whether $e_i = (x_i, t_i)$ is present the sequence used to learn the learning algorithm $\mk M$ by observing the output of $\mk M$ (i.e. the estimates of the Hawkes parameters). 


%
The formulation of Hawkes process \ref{hawkes_def} implies that the occurrence of any event $j$ may be a result of  an \textit{immigration} (due to the background Poisson process) or birth of a child of any of the previous events (indigenous excitation), with probabilities depending on $\mu$, $\alpha$, $\beta$, and the time passed since the previous events. In our setting, it is reasonable to think of a parent-child relation between two events in the event sequence as a result of some type of real-world relation between the two individuals associated with those events. Let us define a \emph{cluster} as a group of individuals whose events forms a \emph{tree} of parent-child relations.
If a cluster is known to an adversary, the presence of a cluster member's data in the sequence may be revealed  by the presence of not only the event directly associated with them, but also the events associated with other individuals in the same cluster. 


In light of this discussion, let us  define \emph{neighboring sequences} of events in this setting, which will inform the adoption of the appropriate notion of privacy for sequential events data problems.
\begin{definition} \label{def_neighbouring_sequences}
Consider sequences of events in form of $S(t) = \{(x_i,t_i)| t_i<t\}$. 
Two sequences of events $S(t)$ and $S_{-j}(t)$ drawn from an unknown  point process observed up until some $t>0$ are ``neighboring'' if they differ only in the presence of 
the events that belong to the same cluster as event $j$.
\end{definition}

In this paper we consider two scenarios with regards to the knowledge of the clusters. In the first scenario, we assume that \emph{i)} the adversary has knowledge of the relations among the individuals whose events is present in the data and \emph{ii)} the learner is aware of the relations or at least the maximum size of clusters of connected individuals (equivalently, maximum tree size in the sequence). In this scenario we say the learner is \emph{relation-aware}. By contrast, in the second scenario, the learner is \emph{relation-unaware} in that it does not have full knowledge of the relations and therefore does not know the maximum size of the clusters. However, in this scenario we assume that the adversary may have greater knowledge of the relations compared to the learner, either now or at some point of time in future.
Providing privacy guarantees in the second scenario requires obtaining probabilistic upper bounds on the size of the event trees in a sequence generated by a Hawkes process. 
We provide such bounds by viewing Hawkes processes as branching processes which results in the immigration-birth representation of Hawkes processes \cite{laub2015hawkes}. 
We state this result in Lemma \ref{lemma_trees_bound} which is an immediate consequence of properties of branching processes. See Section \ref{appendix_lemma_trees} in the supplementary materials for proof.





\begin{lemma}[] \label{lemma_trees_bound}
Consider a Hawkes process $H(t)$ defined by intensity function \ref{hawkes_def} observed until time $T$. For any $0<\gamma\leq 1$ and 
$
    T\geq \left( \frac{\mu\cdot e^2}{\gamma} \right)^{5/2}
$,
with probability at least $1-\gamma$, all existing trees contain at most $\frac{3 \log T}{(1-\alpha)^2}$ individuals. 
\end{lemma}

To account for the fact that the upper bound on the tree sizes in the sequence is probabilistic over the realizations of the underlying Hawkes process, we adopt random DP as the notion of privacy\footnote{Note that the standard DP definition can be recovered from the definition of random DP when $\gamma=0$, accounting for the first scenario with relation-aware learner.}. We formalize random DP for sequential events data generated from a point process as follows.


\begin{definition}\label{def_random_DP}
Let $\mc S^T_{\mbf P}$ be the set of all possible realizations of a temporal point process $\mbf P(t)$ 
until time $T$. A randomized mechanism $\mk{M}: \mc S^T_{\mbf P} \to\mc Y$ is $(\epsilon,\gamma)$-randomly differentially private if
\[
    \P\Big( \forall \mc C\subset\mc Y, \quad \mathbb{P}\left(\mk M(S(T))\in \mc C \right)\leq e^{\epsilon} \P\left(\mk M(S_{-i}(T))\in \mc C \right) \Big)\geq 1-\gamma
\]
where the inner probability is over the randomness of the mechanism, and the outer probability is over neighboring streams $ S(T),  S_{-i}(T)\in \mc S^t_{\mbf P}$ drawn from point process $\mbf P(t)$ until time $T$.
\end{definition} 







\section{Private Estimation parameters of Hawkes process}\label{sect_estimating_params_private}

Our DP estimator is based on the Laplace mechanism \cite{dwork2014algorithmic}, i.e. sampling from a Laplace distribution centered at the output of the non-private estimates. To calculate the variance of the Laplace distribution, one needs to calculate the sensitivity of the estimates $\hat \mu$ and $\hat\alpha$ with respect to the change in any two neighboring data streams generated by a Hawkes process with parameters $\mu$ and $\alpha$.

As mentioned above, the Hawkes parameters estimates $\hat \mu$ and $\hat\alpha$ can be uniquely calculated from the sample mean and sample variance of $Y_i(\Delta)$. Unfortunately, obtaining closed form expressions for $\hat\mu$ and $\hat\alpha$ in terms of $\hat\eta$ and $\hat\sigma^2$ is difficult, which makes it difficult to find the sensitivity of $\hat\mu$ and $\hat\alpha$ with respect to the input sequence. 
Due to post processing immunity of DP \cite{dwork2014algorithmic} and consequently random DP, if we provide privacy guarantees for the $(\hat \mu, \hat\alpha)$, same guarantees hold for $(\hat\eta,\hat\sigma^2)$. Hence, we focus on obtaining privacy guarantees for $\hat\eta$ and $\hat\sigma^2$.

\paragraph{Sensitivity}


The sensitivity of estimates $\hat\eta$ and $\hat\sigma^2$ with respect to the change in the input sequence depends on $\mu$ and $\alpha$. Although we do not have access to $\mu$ and $\alpha$, we assume we know lower and upper bounds on their values: 
%
\begin{align} \label{eq_params_bounds}
    0<\mu_{\text{lower}}\leq\mu \leq \mu_{\text{upper}}, \qquad
    0<\alpha_{\text{lower}}\leq \alpha \leq \alpha_{\text{upper}} <1. 
\end{align}
The following lemma states the sensitivity of $\hat \eta$ and $\hat \sigma^2$ with respect to removing $B$ events from the sequence. The proof of this lemma can be found in Section \ref{appendix_sensitivity_proofs} of the supplementary materials.

\begin{lemma}\label{lemma_sensitivity}
Consider the count series $\{Y_i(\Delta)\}_{i=1}^K$ associated with a Hawkes process $\mbf H$ defined by intensity function \ref{hawkes_def}. Suppose that the maximum number of correlated events in the Hawkes process is $B$. Then, the maximal amount of change in the sample mean $\hat\eta$ of the counts is at most  $\frac{B}{K}$. Moreover, with probability at least $1-\gamma$, the maximal amount of change in the sample variance is upper bounded by $\frac{B^2}{K} + \frac{2B^{3/2}\sqrt{\Delta}C_{1}}{(K-1)}$ where $C_{1} = \sqrt{\frac{1.1\cdot\mu_{\text{upper}}}{(1-\alpha_{\text{upper}})^3}\cdot\frac{1}{\gamma}}$.
\end{lemma}





As previously discussed, the maximal number of events that distinguish two neighbouring sequences corresponds to the number of individuals in the largest tree in branching process representation. In the scenario with relation-unaware learner where we do not know $B$, recall that Lemma \ref{lemma_trees_bound} provides probabilistic upper bound on the size of the largest tree.

\begin{corollary}\label{cor_sensitivity_relation_unaware}
As a consequence of Lemma \ref{lemma_trees_bound} and Lemma \ref{lemma_sensitivity}, for $0<\gamma\leq 1/2$ and $T$ such that $T\geq \left( {\mu\cdot e^2}/{\gamma} \right)^{5/2},$ in the scenario with relation-unaware learner, sensitivity of the sample mean and variance are $\frac{C_{2}\log T}{K}$ and $\frac{(C_{2}\log T)^2}{K} + \frac{2(C_{2}\log T)^{3/2}\sqrt{\Delta}C_{1}}{(K-1)}$, respectively with probability at least $1-2\gamma$, where $C_{2} = \frac{3}{(1-\alpha_{\text{upper}})^2}$.
\end{corollary}

\paragraph{DP sample mean and sample variance}
For a given $\epsilon>0$, we introduce the following DP version of the sample mean and the sample variance for the sequence observed until time $T=K\Delta$.
\begin{align}
    &\hat{\eta}_{\text{private}} = \hat{\eta}+\Lambda\Big(\frac{C_{2}\log{T}}{ {K}\cdot \epsilon}\Big)\label{eq_DP_mean}\\
    &{\hat{\sigma}^{2}_{\text{private}}}
    = {\hat{\sigma}^{2}}+\Lambda\Big( \frac{C_{2}^{2}(\log T)^2 + 2C_{2}^{3/2}C_{1}\cdot\frac{K}{K-1}(\log T)^{3/2}\sqrt{\Delta}}{K\cdot\epsilon}\Big)\label{eq_DP_var}
\end{align}

\begin{lemma}\label{lemma_random_DP_mean_var}
The noisy mean estimator $\hat{\eta}_{\text{private}}$ is $(2\gamma,\epsilon)$-random-DP. Moreover, the noisy variance estimator ${\hat{\sigma}^{2}_{\text{private}}}$ is $(2\gamma,\epsilon)$-random-DP.
\end{lemma}

Proof of Lemma \ref{lemma_random_DP_mean_var} (see Section \ref{appendix_random_DP} in the supplementary material) relies on Corollary \ref{cor_sensitivity_relation_unaware} and standard results on Laplace mechanism. 




\paragraph{DP estimation of the Hawkes parameters}
As previously discussed, $\hat\eta$ and $\hat{\sigma}^{2}$ uniquely define $\hat{\mu}$ and $\hat{\alpha}$ from system of equations \eqref{eq_stationary_mean_hawkes} and \eqref{eq_stationary_var_hawkes}. 

\begin{theorem}\label{thm_alpha_mu_random_DP}
Let $\hat{\mu}_{\text{private}}$ and $\hat{\alpha}_{\text{private}}$ be the estimators for $\mu$ and $\alpha$ by solving system of equations \eqref{eq_stationary_mean_hawkes} and \eqref{eq_stationary_var_hawkes} with private estimates $\hat{\eta}_{\text{private}}$ and ${\hat{\sigma}^{2}_{\text{private}}}$. Then, this estimation mechanism is  $(2\gamma,2\epsilon)$-random-DP.
\end{theorem}
Theorem \ref{thm_alpha_mu_random_DP} is a consequence of the standard composition theorem \cite{kairouz2015composition} and the post processing immunity \cite{dwork2014algorithmic} of differential privacy. See Section \ref{appendix_random_DP} in the supplementary material.

\paragraph{Remark.} In the scenario with relation aware learner where $B$ is known, the term $C_2 \log T$ in equations \eqref{eq_DP_mean} and \eqref{eq_DP_var}
will be replaced by the known value $B$. Also the random privacy guarantees in Lemma \ref{lemma_random_DP_mean_var} and Theorem \ref{thm_alpha_mu_random_DP} will improve to $(\gamma, \epsilon)$ and $(\gamma, 2\epsilon)$, respectively.

\subsection{Sample complexity of the private estimator}

To provide a concrete privacy utility trade-off analysis of the DP estimation procedure presented in this section, we obtain sample complexity upper bounds for the private estimator described by equations \eqref{eq_DP_mean} and  \eqref{eq_DP_var}. By comparing these results with those of the non-private estimator, stated in Theorem \ref{thm_non_private_params}, one can quantify the cost of making our estimator differentially private in terms of the additional observation time required by the private estimator to reach within a desired distance of the true Hawkes parameters. 
To the best of our knowledge our analysis is the first to formalize the privacy utility trade-off for DP learning of Hawkes processes.

\begin{theorem}\label{thm_private_params}
Consider the same setup and conditions as in Theorem \ref{thm_non_private_params}, except with private estimates $\hat{\mu}_{\text{private}}$ and $\hat{\alpha}_{\text{private}}$ as described in equations \eqref{eq_DP_mean} and \eqref{eq_DP_var}. In addition, let $C_{1} = \sqrt{\frac{1.1\cdot\mu_{\text{upper}}}{(1-\alpha_{\text{upper}})^3}\cdot\frac{1}{\gamma}}$ and $C_{2} = \frac{3}{(1-\alpha_{\text{upper}})^2}$. For the choice of $\Delta =c \log T$ for some constant $c$, if
\begin{align}
&T\geq \frac{C_{9}^{2}\mu_{\text{upper}} }{(1-\alpha_{\text{upper}})^3\xi^2} \max\Big\{ \Psi(1-\frac{\delta}{16})^{2},~
9C_{9}^{2}\Psi(1- \frac{\delta}{32})^2(\eta_4 - \sigma^2) \Big\} \quad \text{and}\label{condition_dp_1}\\
 & \frac{T}{\log T}\geq \frac{c\mu_{\text{upper}} }{(1-\alpha_{\text{upper}})^3\xi} \max\Big\{
 3C_{9}\Psi( 1-\frac{\delta}{32}) ^2, ~
 \frac{48C_{9}}{\delta}\Big\}\quad \text{and}\label{condition_dp_2}\\
 &\frac{T}{(\log T)^{5/2}}>\frac{4\sqrt{c}C_{1}C_{2}^2C_{9}}{\epsilon\xi}\log \left(\frac{4}{\delta}\right)\label{condition_dp_3}
\end{align}
%




%
%
for some $0<\delta\leq 1$, then we have $\mathbb{P}\left(\left\vert\hat{\mu}_{\text{private}}-\mu\right\vert>\xi\right) \leq \delta$ and $\mathbb{P}\left(\left\vert\hat{\alpha}_{\text{private}}-\alpha\right\vert>\xi\right) \leq \delta$ for $(\gamma, 2\epsilon)$-random DP estimates $\hat{\mu}_{\text{private}}$ and $\hat{\alpha}_{\text{private}}$. 
\end{theorem}
\begin{proof}[Proof sketch]
Theorem \ref{thm_non_private_params} give us lower bound on $T$ which yield $\hat{\mu}$ and $\hat{\alpha}$ being within desired distance of $\mu$ and $\alpha$ respectively. We want Theorem \ref{thm_non_private_params} to hold with $\frac{\xi}{2}$ and  $\frac{\delta}{2}$, and so we modify condition on $T$ to reflect this budget. 
The remaining $\frac{\xi}{2}, \frac{\delta}{2}$ budget is used to bound distance between sample estimates and their private counterparts, and relies on the tail bound of Laplace random variable. We show that conditions \eqref{condition_dp_1} and \eqref{condition_dp_2} are sufficient for the non-private estimates to be within $\xi/2$ distance of the true values with $\Delta = c\log T$ with probability at least $1-\delta/2$. We then show that condition \eqref{condition_dp_3} which is required to bound the tail of Laplace distribution 
(recall (\ref{eq_DP_mean}) and (\ref{eq_DP_var})) guarantees that the private estimates are within $\xi/2$ distance of the non-private estimates with probability at least $1-\delta/2$. See Section \ref{appendix_private_estimates} in the supplementary material for the detailed proof of Theorem \ref{thm_private_params}.
%
%
\end{proof}


\paragraph{Remark.}
For the relations aware scenario with the maximum number of correlated events $B$, condition (\ref{condition_dp_3}) becomes
$\frac{T}{\sqrt{\log T}}>\frac{4\sqrt{c}B^2 C_{1}C_{9}}{\epsilon\xi}\log \left(\frac{4}{\delta}\right)
$. 

\subsection{Cost of privacy} 
%
For the inverse CDF function $\Psi(\cdot)$, we have $\lim_{x\to 0}\Psi(1-x) =\sqrt{2 \log \frac{1}{x}} $ \cite{blair1976rational}. Therefore, for arbitrarily small $\delta$ and large enough $T$, the condition on $T$ in Theorem \ref{thm_non_private_params} becomes $T = O(\frac{1}{\delta\xi})$. In contrast, for small enough $\delta$ and large enough $T$,   for the private estimator (in the relation unaware scenario) the requirement on $T$ is $\frac{T}{\log T} = O(\frac{1}{\xi\delta})$ which, ignoring $\log\log$ terms, can be written as $T=O(\frac{\log(1/\delta\xi)}{\delta\xi})$. 
Equivalently, one can state these results in terms of how fast the estimates converge to the true values $(\mu,\alpha)$. While the non-private estimates $(\hat\mu,\hat\alpha)$ have a convergence rate of $O(\frac{1}{T})$, the rate of convergence for the private estimates $(\hat\mu_{\text{private}},\hat\alpha_{\text{private}})$ is $O(\frac{\log T}{T})$. This means that the DP guarantee of the private estimator comes at the cost of an extra $\log T$ term (as well as larger constants, i.e. $48$ instead of $24$ in the dominant term) in the convergence rate. 
Furthermore, in the relation aware scenario, the convergence rate $O(\frac{1}{T})$ of the non-private estimator is preserved, however with larger constants ($48$ instead of $24$).




\section{Experiments}\label{sect_experiments}
In this section we present numerical experiments to assess the privacy-utility tradeoff of the differentially private estimator introduced in Section~\ref{sect_estimating_params}. We define the utility as the numerical error in estimating the $\mu$ and $\alpha$ parameters. We first add noise to the estimate of the mean and variance of interval event counts according to equations~\eqref{eq_DP_mean} and \eqref{eq_DP_var}. We then obtain an estimate of the Hawkes process parameters by numerically solving the system of two equations \eqref{eq_stationary_mean_hawkes} and \eqref{eq_stationary_var_hawkes}. In all experiments, we set $\alpha_{\rm upper}=0.75$, $\mu_{\rm upper}=2.0$, $\gamma=0.05$ and the Hawkes process decay $\beta=1.0$. The interval length $\Delta$ is set as $\mathcal{O}(\log T)$ rounded to the closest multiple of $5$.

\textbf{Synthetic Data.} We simulate timestamps from two different Hawkes processes (using the \texttt{tick} Python package \cite{bacry2017tick}), with the following two sets of parameters: $(\mu_1, \alpha_1)=(1, 0.5)$ and $(\mu_2, \alpha_2)=(1.5, 0.3)$. We let the simulation run for $T=100,000$ units of time, with a interval size $\Delta=10$. Figure~\ref{fig:pu_synth} shows the privacy-utility trade-off for the differentially private estimator for $\alpha_1$ and $\alpha_2$ respectively, with the utility presented as the normalized absolute estimation error (see Supplementary Material for the same figure for $\mu_1$ and $\mu_2$). We explore the trade-off by varying the maximum tree length $B$ to $10, 25$ and $100$, as well as setting $B=C_2 \log T$ as in Corollary~\ref{cor_sensitivity_relation_unaware}. We append to both figure the non-private estimate (indicated by a vertical black line), which indeed achieves a negligible estimation error. The colored bands represent the $95\%$ bounds, obtained over 50 repetitions per each privacy level $\epsilon$. As expected, we see the utility increasing (smaller estimation error) with the privacy budget increasing. For a given privacy budget, the utility decreases the larger the maximum tree length $B$. In addition, we have observed the estimation process failing to converge if the privacy budget is too small, which is the reason behind the missing trend for $B \approx 100$. In Figure~\ref{fig:time_synth}, we instead explore how long the simulation needs to run to achieve an error below a certain threshold, for a given privacy budget $\epsilon$ and across different maximum tree lengths $B$. We set the error threshold to $10\%$ of $\alpha_1$ and $\alpha_2$ respectively, and take the median running time over 10 repetitions per each $\epsilon$ value. We observe that the running time required to achieve a specific error level (a) decreases with the privacy budget but (b) increases with longer trees due to the fact that more events are affected by a single immigration or birth.

\textbf{Real Data.} We consider the following real event sequence datasets: (1) \textit{MathOverflow} \cite{leskovec2014snap}, a dataset of user interactions in a question-answering website focused on mathematics from 2009 to 2016. We select three of the most active users in terms of interactions, grouping together all the user interactions with the website, (2) \textit{MOOC} \cite{feng2019dropout}, a dataset of student actions in a massive open online course at a Chinese university. Again we select $3$ of the most active users in terms of interactions with the open course, grouping all actions for the entire online course and (3) \textit{911 Calls}\footnote{Data available at \url{https://www.kaggle.com/datasets/mchirico/montcoalert}}, a dataset of emergency calls in Montgomery, PA, from 2015 to 2020. We filter out the medical emergency calls and consider only fire and traffic-related emergency calls, for a which a self-excitation dynamic is more plausible. 
We use the non-private estimates of the Hawkes process parameters as ground truth, as the true generating parameters are not available. Figure~\ref{fig:pu_real} shows the privacy utility trade-off across all real datasets listed above, with the utility being reported as the normalized absolute error for the $\mu$ parameter. We set the maximum tree length to $B=10$, as higher values were leading to estimation issues. We report the non-private estimates of $(\mu, \alpha)$ as part of the legend for each event sequence. As in the synthetic data case, a larger privacy budget corresponds to a lower estimation error. However, the slope of the trade-off is not uniform, seemingly steeper when the magnitude between $\alpha$ and $\mu$ is different (911 Calls and MOOC) and flatter when the two parameters are comparable (MathOverflow).

\begin{figure}[!ht]
    \centering
    \includegraphics[width=1.0\textwidth]{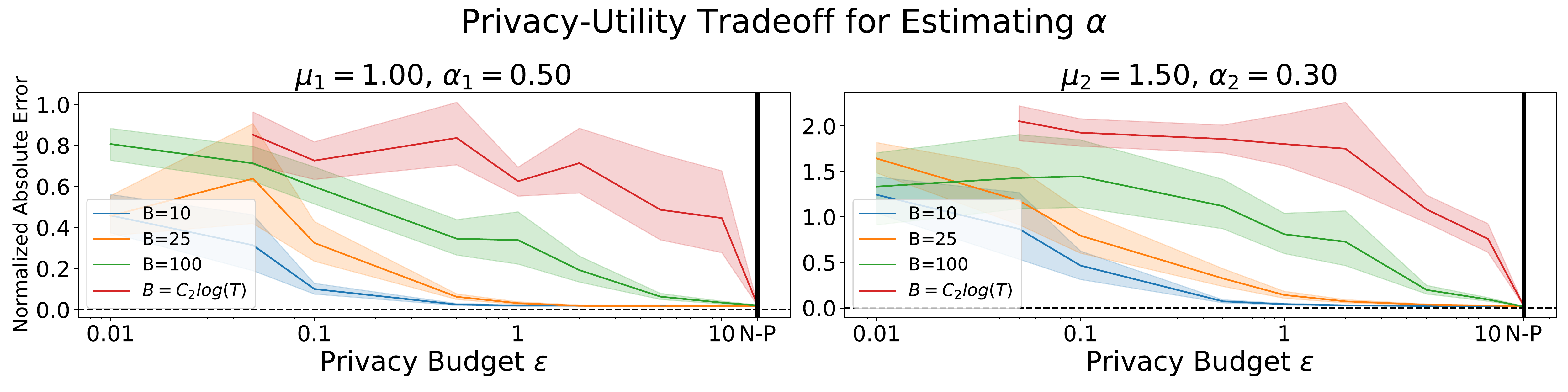}
    \caption{Privacy-utility trade-off for our differentially private estimator for $\alpha_1$ (left) and $\alpha_2$ (right). Mean and $95\%$ bands observed over 50 repetitions at each privacy level, along with the non-private estimation error appended at the rightmost part of the x-axis. Utility increases with a larger privacy budget, and decreases with the maximum tree length $B$. See text for more details.  }
    \label{fig:pu_synth}
\end{figure}

\begin{figure}[!ht]
    \centering
    \includegraphics[width=1.0\textwidth]{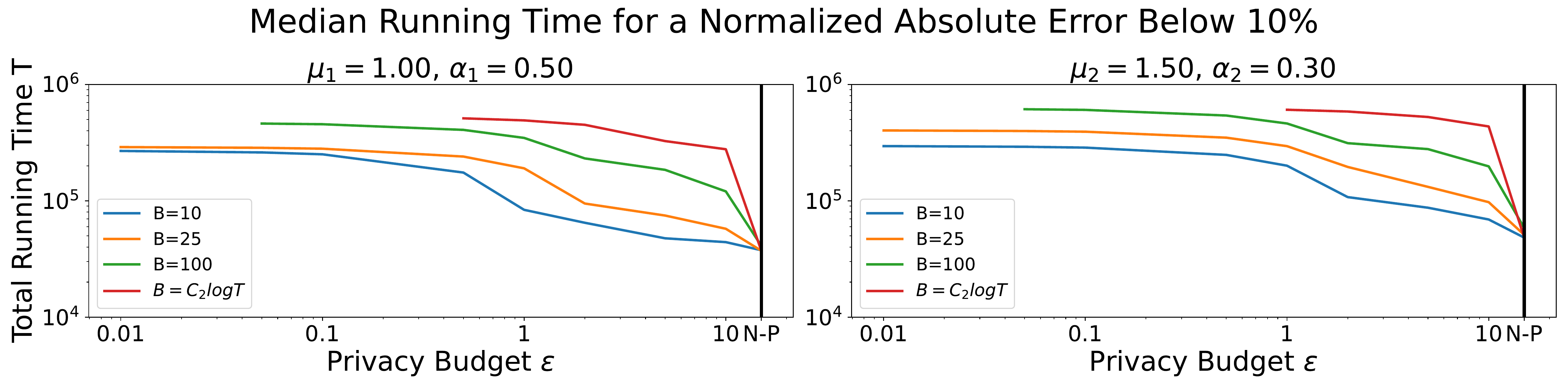}
    \caption{Running time to achieve an estimation error below $10\%$ for $\alpha_1$ (left) and $\alpha_2$ (right). Median time over 10 repetitions per privacy budget is reported. Higher privacy budgets achieve the same estimation errors faster, while longer trees in Hawkes process events require a longer running time.}
    \label{fig:time_synth}
\end{figure}

\begin{figure}[!ht]
    \centering
    \includegraphics[width=1.0\textwidth]{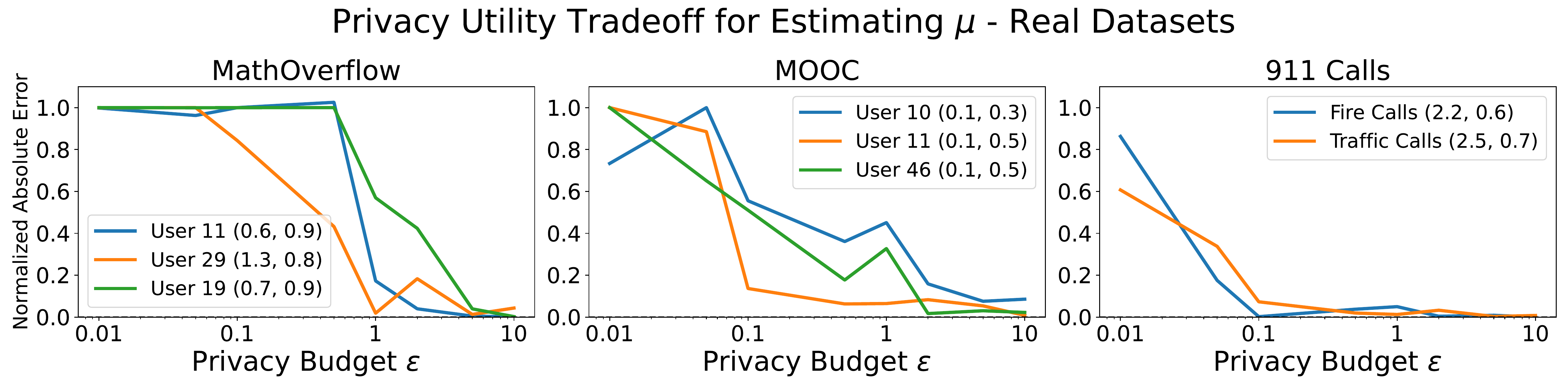}
    \caption{Privacy-utility tradeoff for our differentially private estimator for $\mu$ in all real datasets. Ground truth $(\mu, \alpha)$ parameters indicated in the legend for each event sequence data. Utility increases as privacy budget increases, with the increase being flatter if $\mu$ and $\alpha$ have similar magnitudes magnitude (MathOverflow) and steeper otherwise (MOOC and 911 Calls). See text for more details.}
    \label{fig:pu_real}
\end{figure}


\section{Conclusion and Discussion}\label{sect_conclusion}
In this work we provide sample complexity results for estimating the parameters of a Hawkes process with an immigration rate $\mu>0$ and exponential kernel $\alpha e^{-\beta t}$ in the regime $\alpha < 1$. We present a differentially private version of the estimates within the settings of random differential privacy, along with sample complexity results. This allows us to analyse the cost of adding privacy to the estimates. 
Specifically, we showed the the observation time $T$ required to reach within $\xi$ distance of the true parameters with probability at least $1-\delta$ increases from $T=O(\frac{1}{\delta\xi})$ for the non-private estimator to $T=O(\frac{\log(1/\delta\xi)}{\delta\xi})$ for the private estimator.
In future work, we plan on extending such results to $D > 1$ event types, where $\alpha \in \mathbb{R}_+^{D \times D}$ is the adjacency matrix describing the self and cross-excitation terms. 
We also plan to study differentially private estimators for avoiding leakage of entire sequences rather than single events, which is common in domains such as advertising or finance where sequences represent set of actions for specific customers.


\textbf{Disclaimer.} This paper was prepared for informational purposes by the Artificial Intelligence Research group of JPMorgan Chase \& Co and its affiliates (“J.P. Morgan”), and is not a product of the Research Department of J.P. Morgan.  J.P. Morgan makes no representation and warranty whatsoever and disclaims all liability, for the completeness, accuracy or reliability of the information contained herein.  This document is not intended as investment research or investment advice, or a recommendation, offer or solicitation for the purchase or sale of any security, financial instrument, financial product or service, or to be used in any way for evaluating the merits of participating in any transaction, and shall not constitute a solicitation under any jurisdiction or to any person, if such solicitation under such jurisdiction or to such person would be unlawful.  









\bibliographystyle{abbrv}

\newpage

\appendix

\section{Appendix}
\subsection{Progeny of the largest tree: proof of Lemma \ref{lemma_trees_bound}}\label{appendix_lemma_trees}






We first highlight some well known results on Galton-Watson branching processes with Poisson offspring distribution. We rely on presentation from Chapter 6 of \cite{rochLectureNotes}.






\begin{lemma}
Let $W$ be the total progeny of Galton-Watson branching process with Poisson offspring distribution with mean $\nu$, $\nu\geq 0$. For integers $j\geq 1$, we have
\begin{align*}
    \mathbb{P}\left(W=j\right) = e^{-\nu j}\frac{(\nu j)^{j-1}}{j!}
\end{align*}
\end{lemma}

In particular, we have $E(W)=\frac{1}{1-\nu}$. As a straightforward consequence, we have the following result on the tail of $W$. We provide proof for completeness.




\begin{lemma}\label{lemma_app_poisson_progeny}
Let $W$ be the total progeny of Poisson branching process with parameter $\nu<1$, and $D>\mathbb{E}(W) = \frac{1}{1-\nu}$. We have
\begin{align}
    \mathbb{P}\left(W>D\right)\leq e^2 \exp(-I_{\nu} D).
\end{align}
where $0<I_{\nu} = \nu-1-\log \nu.$
\end{lemma}


\begin{proof}
We have
\begin{align}
   \mathbb{P}\left(W>D\right) = \sum_{j>D}\mathbb{P}\left( W=j \right) &= \sum_{j>D}\mathbb{P}\left( e^{-\nu j}\frac{(\nu j)^{j-1}}{j!} \right)\\
   & \leq \sum_{j>D} \frac{e^{-\nu j}\cdot \nu^{j-1}\cdot j^{j-1} \cdot e^{j}}{\sqrt{2\pi j}\cdot j^j \exp(\frac{1}{12j+1})} \label{ineq_stirling_tree} \\ 
   & = \sum_{j>D} \frac{e^{-(j-1)(\nu-\log\nu-1)}\cdot{e^{1-\nu}}}{\sqrt{2\pi}\cdot j^{3/2}\cdot \exp(\frac{1}{12j+1})}\\
   & \leq e^{2} \cdot \sum_{j>D} e^{-(j-1)(\nu-\log\nu-1)}\\
   & =  e^{2} \cdot \sum_{j\geq D} (e^{-(\nu-\log\nu-1)})^{j}\\
   & = e^{2}\exp(-I_{\nu}D) \label{eq_branching_progeny_last}
 \end{align}
where (\ref{ineq_stirling_tree}) follows by Stirling formula, and (\ref{eq_branching_progeny_last}) is the consequence of $I_{\nu}<1$ for $\nu<1$.
\end{proof}

\begin{corollary}\label{remark_app_progeny_Poiss_branching}
In particular, for the total progeny of Poisson branching process with mean $\nu<1$, and $a>1$ we have
\begin{align}
    \mathbb{P}\left(W> \frac{a}{1-\nu}\right)\leq e^2\exp\left(-\frac{10a}{21}(1-\nu)\right).
\end{align}
\end{corollary}
\begin{proof}
This follows from the Lemma \ref{lemma_app_poisson_progeny} for $D=\frac{a}{1-\nu}$ by
\begin{align}
    \mathbb{P}\left(W> \frac{a}{1-\nu}\right) & \leq e^{2}\exp\left(-I_{\nu}\cdot \frac{a}{1-\nu}\right)\\
    & = e^{2}\exp\left(-(\nu-1-\log \nu)\cdot \frac{a}{1-\nu}\right)\\
    & = e^{2}\exp\left((1-\nu+\log \nu)\cdot \frac{a}{1-\nu}\right)\\
    & = e^{2} \exp\left(a\left(1+\frac{\log \nu}{1-\nu}\right)\right)\\
    & \leq e^{2}\exp\left(a\left(-\frac{1}{2}(1-\nu)+o(1-\nu)\right)\right)\\
    & \leq e^{2} \exp\left(-\frac{10a}{21}(1-\nu)\right)
    \end{align}
where we recall that $I_{\nu} = \nu-1-\log \nu$ and for $\nu<1$ by Taylor expansion
\begin{align}
    \log \nu = \log(1+\nu-1) = (\nu-1)-\frac{1}{2}(\nu-1)^2+o(\vert\nu-1\vert^2)
\end{align}
and thus
\begin{align}
    1+\frac{\log\nu}{1-\nu} = -\frac{1}{2}(1-\nu)+o(\vert\nu-1\vert)
\end{align}
\end{proof}


In the branching process representation of Hawkes processes, an event generated by an  excitation is seen as an offspring i.e. a direct child of another event, whose excitation led to the occurrence of the event in question. Moreover, every event generated by excitation has an ancestor event (not necessarily a direct parent) that arrived from the background Poisson process. In this heuristic, any event that has occurred due to the exogenous component (i.e. an immigrant from the background Poisson process) is the root of a tree. In the following lemma, we state the offspring (i.e. direct children) distribution of the arbitrary arrival of Hawkes process defined by (\ref{hawkes_def}) with excitation rate $\alpha<1$ and decay rate $\beta=1$.

\begin{lemma}[Proposition 3.1. of \cite{daw2018ephemerally}.]\label{lemma_app_Hawkes_offspring}
Consider the Hawkes process defined by (\ref{hawkes_def}) with excitation rate $\alpha<1$ and decay rate $\beta=1$. For an arbitrary arrival of a Hawkes process, let $Z$ denote the number of its direct offsprings, i.e. events occurring by a direct excitation of its arrival. Then, $Z$ has Poisson distribution with mean $\alpha$, i.e. for $k\geq 0$ we have
\begin{align}
    \mathbb{P}\left(Z=k\right) = \frac{e^{-{\alpha}}}{k!}\alpha^k.
\end{align}
\end{lemma}

\begin{corollary}\label{remark_app_Hawkes_progeny_tail}
For the Hawkes process defined by (\ref{hawkes_def}) with excitation rate $\alpha<1$ and decay rate $\beta=1$, let $W$ denote total progeny of a tree rooted in an arbitrary exogenous arrival (i.e. an immigrant of the background Poisson process). As a consequence of Corollary \ref{remark_app_progeny_Poiss_branching} and Lemma \ref{lemma_app_Hawkes_offspring} for $a>1$, we have
\begin{align}
    \mathbb{P}\left(W> \frac{a}{1-\alpha}\right)\leq e^2\exp\left(-\frac{10a}{21}(1-\alpha)\right)
\end{align}
\end{corollary}

So far, we studied the progeny of any tree in the branching process representation of a given Hawkes process. Next, we study the progeny of the largest tree observed until time $T$, where the size of a tree is defined as the number of events belonging to the tree.



\begin{lemma}\label{lemma_bound_all-trees}
In the same setting as in Corollary \ref{remark_app_Hawkes_progeny_tail}, the probability that the largest tree observed until time $T$ contains at most $\frac{a}{1-\alpha}$ individuals for $a>1$ is at least
\begin{align}
    \exp\left(-\mu T e^2 \cdot \exp(-\frac{10a}{21}(1-\alpha)) \right).
\end{align}
\end{lemma}
\begin{proof}
Let $\Upsilon(T)$ denote the number of trees observed until time $T$, and for $i=1,\dots,\Upsilon(T)$ let $W(i)$ denote the progeny of tree $i$, i.e. the total number of individuals tree $i$ contains. Thus, the probability that the largest tree contains at most $\frac{a}{1-\alpha}$ individuals equals
\begin{align}
    \sum_{k=0}^{\infty}&\mathbb{P}\left(\Upsilon(T) = k, W(1)<\frac{a}{1-\alpha},\dots,W(\Upsilon(T))<\frac{a}{1-\alpha}\right)\\
    & = \sum_{k=0}^{\infty}\mathbb{P}\left(W(1)<\frac{a}{1-\alpha},\dots,W(\Upsilon(T))<\frac{a}{1-\alpha} \Big| \Upsilon(T) = k \right)\cdot\mathbb{P}\left( \Upsilon(T) = k \right)\\
    & = \sum_{k=0}^{\infty} \mathbb{P}\left(W(1)<\frac{a}{1-\alpha} \Big| \Upsilon(T) = k \right)^k\cdot\mathbb{P}\left( \Upsilon(T) = k \right)\label{eq_app_trees_indep}\\
     & = \sum_{k=0}^{\infty} \mathbb{P}\left(W(1)<\frac{a}{1-\alpha} \Big| \Upsilon(T) = k \right)^k\cdot \frac{(\mu T)^k e^{-\mu T}}{k!}\label{eq_app_num_trees_poiss}\\
    & \geq \sum_{k=0}^{\infty}\left( 1- e^2\exp\left(-\frac{10a}{21}(1-\alpha)\right) \right)^k \cdot \frac{(\mu T)^k e^{-\mu T}}{k!}\label{eq_app_all_trees_less}\\
    & = \exp\left(-\mu T \cdot e^2 \cdot \exp(-\frac{10a}{21}(1-\alpha)) \right)
\end{align}
where (\ref{eq_app_trees_indep}) holds as progenies $W(i)$ are iid and (\ref{eq_app_num_trees_poiss}) is the consequence of the fact that the number of trees $\Upsilon(T)$ corresponds to the number of arrivals in the background Poisson process up until $T$. Finally, (\ref{eq_app_all_trees_less}) follows by Corollary \ref{remark_app_Hawkes_progeny_tail}. This concludes the proof.
\end{proof}

\begin{corollary}\label{corr_app_largest_tree}
For $0<\gamma<1$, probability that the largest tree we observe up until time $T$ contains at most
\begin{align*}
    \frac{2.1}{(1-\alpha)^2}\left( \log \left( \mu e^2 T \cdot \frac{1}{\gamma} \right) \right)
\end{align*}
individuals is at least $1-\gamma$. This is a direct consequence of Lemma \ref{lemma_bound_all-trees} for $a=\frac{2.1}{1-\alpha}\left( \log \left( \mu e^2 T \cdot \frac{1}{\gamma} \right) \right)$, as
\begin{align*}
  \exp\left(-\mu T e^2 \cdot \exp(-\frac{10a}{21}(1-\alpha)) \right) & = \exp\left(-\mu T e^2 \exp\left(-\log \left( \mu e^2 T \cdot \frac{1}{\gamma} \right) \right) \right)\\
  &=e^{-\gamma}\\
  &\geq 1-\gamma
\end{align*}
\end{corollary}

Finally, we are ready to prove Lemma \ref{lemma_trees_bound}.
\begin{proof}[Proof of Lemma \ref{lemma_trees_bound}]
By Corollary \ref{corr_app_largest_tree}, for $0<\gamma\leq 1$, the largest tree we observe until time $T$ contains at most
\begin{align}\label{eq_tree_bound_appendix}
    \frac{2.1}{(1-\alpha)^2}\left( \log\left(\mu e^2  T \cdot \frac{1}{\gamma} \right) \right)
\end{align}
individuals
with probability at least $1-\gamma$. Having in mind the assumption $T\geq \left(\frac{\mu e^{2}}{\gamma}\right)^{5/2}$, we have
\begin{align}\label{eq_tree_bound_appendix2}
    \frac{2.1}{(1-\alpha)^2}\left( \log\left(\mu e^2  T \cdot \frac{1}{\gamma} \right) \right)&= \frac{2.1}{(1-\alpha)^2}\left(\log T+\log \left(\frac{\mu e^2}{\gamma}\right)\right)\\
    & \leq \frac{2.1\cdot1.4}{(1-\alpha)^2}\log T\\
    & \leq \frac{3}{(1-\alpha)^2}\log T.
\end{align}
Hence, with probability at least $1-\gamma$, the largest tree contains at most $\frac{3\log T}{(1-\alpha)^2}$ individuals.
\end{proof}




\subsection{Detailed calculation of the mean and variance of $Y^{\Delta}_i$} \label{appendix_calculation_mean_var_Hawkes_stationary}

Proposition 2.2 from \cite{daw2018queues} states that for $\alpha<1$, $\beta=1$ we have
\begin{align}
    \mathbb{E}(H_{t}) =& \lambda_{\infty}t+ \frac{\lambda_{0}-\lambda_{\infty}}{1-\alpha}\left(1-e^{-(1-\alpha)t}\right),\\
    \mathbf{Var}(H_{t}) =& \frac{\lambda_{\infty}}{(1-\alpha)^2}t+\frac{\alpha^2(2\lambda_{0}-\lambda_{\infty})}{2(1-\alpha)^3}\left(1-e^{-2(1-\alpha)t}\right)-\frac{2\alpha(\lambda_{0}-\lambda_{\infty})}{(1-\alpha)^2}t e^{-(1-\alpha)t}\\
    &+\left( \frac{1+\alpha}{(1-\alpha)^2}\left(\lambda_{0}-\lambda_{\infty}\right)-\frac{2\alpha}{(1-\alpha)^3}\lambda_{\infty} \right)\left(1-e^{-(1-\alpha)t}\right),
\end{align}
and thus for the counting series $\{Y_{i}(\Delta) = \mbf N(i\Delta)- \mbf N((i-1)\Delta) \}_{i=1}^K$ we have

\begin{align}
    \mathbb{E}[Y_i({\Delta})] =& \lambda_{\infty}\Delta+ \frac{\lambda^{*}_{(i-1)\Delta}-\lambda_{\infty}}{1-\alpha}\left(1-e^{-(1-\alpha)\Delta}\right),\\
    \mathbf{Var}[Y_{i}(\Delta)] =& \frac{\lambda_{\infty}}{(1-\alpha)^2}\Delta+\frac{\alpha^2(2\lambda^{*}_{(i-1)\Delta}-\lambda_{\infty})}{2(1-\alpha)^3}\left(1-e^{-2(1-\alpha)\Delta}\right)\\
    &-\frac{2\alpha(\lambda^{*}_{(i-1)\Delta}-\lambda_{\infty})}{(1-\alpha)^2}\Delta e^{-(\beta-\alpha)\Delta}\\
    &+\left( \frac{1+\alpha}{(1-\alpha)^2}\left(\lambda^{*}_{(i-1)\Delta}-\lambda_{\infty}\right)-\frac{2\alpha}{(1-\alpha)^3}\lambda_{\infty} \right)\left(1-e^{-(1-\alpha)\Delta}\right).
\end{align}
The stationarity assumption gives us that $\E[\lambda_{(i-1)\Delta}] = \lambda_{\infty}$. Moreover, we have $\E[Y_i^{\Delta}] = \E_{\lambda_{(i-1)\Delta}}\big[\E[\mbf N(\Delta)]\big ]$, $\E[Y_i^{\Delta}]^2 = \E_{\lambda_{(i-1)\Delta}}\big[\E[\mbf N(\Delta)]^2\big ]$ and $\mathbf{Var}[Y_i^{\Delta}] = \E[Y_i^{\Delta}]^2- (\E[Y_i^{\Delta}])^2$. Bearing in mind $\lambda_{\infty} =\frac{\mu}{1-\alpha}$, we obtain \eqref{eq_stationary_mean_hawkes} and \eqref{eq_stationary_var_hawkes}.






\subsection{Number of events: deviation from the mean for a fixed number of bins}\label{appendix_cum_discrepancy_B_bins}

In this section we provide a probabilistic upper bound on the cumulative discrepancy from the mean for the number of events, i.e. an upper bound on $\sum_{i=1}^{B}\left\vert Y_{i}-\mathbb{E}(Y_{i})\right\vert$. Such a bound will be useful for studying the sensitivity of sample variance for the neighbouring datasets that differ in at most $B$ events.

\begin{lemma} \label{lemma_high_prob_bins_event}
Consider the count series $\{Y_i(\Delta)\}_{i=1}^K$ associated with a Hawkes
process with parameters $\mu_{\text{lower}}<\mu<\mu_{\text{upper}}$ and $\alpha^{\text{lower}}<\alpha<\alpha^{\text{upper}}$ where $\Delta$ is such that
\begin{align}
    \Delta>10\cdot\frac{\alpha^2}{2(1-\alpha)}.
\end{align}
Then, for any set of $B$ bins we have
\begin{align}
\sum_{i=1}^{B}\left\vert Y_{i}-\mathbb{E}(Y_{i})\right\vert < B^{3/2}\sqrt{\Delta}\cdot{C_{1}},
\end{align}
with probability at least $1-\gamma$ where $C_{1} = \sqrt{\frac{1.1\cdot\mu_{\text{upper}}}{(1-\alpha_{\text{upper}})^3}\cdot\frac{1}{\gamma}}$.
\end{lemma}

\begin{proof}
We have
\begin{align}
\mathbb{P}\left(\sum_{i=1}^{B}\left\vert Y_{i}-\mathbb{E}(Y_{i})\right\vert > B^{3/2}\sqrt{\Delta}\cdot{C_{1}}\right) &\leq
B\cdot\mathbb{P}\left( \left\vert Y_{i}-\mathbb{E}(Y_{i})\right\vert \geq\sqrt{B}\sqrt{\Delta}\cdot{C_{1}} \right)\\
&\leq \frac{B\sigma^{2}(\Delta)}{B\Delta \cdot C_{1}^2} \label{eq_app_Chebishev}\\
& \leq \frac{1}{\Delta\cdot{C_{1}}^2}\left( \frac{\mu\cdot\Delta}{(1-\alpha)^3}+\frac{\alpha^2\mu}{2(1-\alpha)^4} \right)\label{eq_app_var_bound}\\
& \leq\frac{1}{\Delta\cdot{C_{1}}^2}\left( \frac{\mu\cdot\Delta}{(1-\alpha)^3}+0.1\cdot \frac{\mu\cdot\Delta}{(1-\alpha)^3} \right) \label{eq_app_Delta_lower_bins}\\
& \leq \frac{1.1}{C_{1}^2}\cdot\frac{\mu}{(1-\alpha)^3}\\
& \leq \frac{1.1}{C_{1}^2}\cdot\frac{\mu_{\text{upper}}}{(1-\alpha_{\text{upper}})^3}\\
&\leq \gamma
\end{align}
where \eqref{eq_app_Chebishev} is the consequence of Chebyshev's inequality, \eqref{eq_app_var_bound} follows from the expression for $\sigma^{2}(\Delta)$ in \eqref{eq_stationary_var_hawkes} and \eqref{eq_app_Delta_lower_bins} follows from the lower bound on $\Delta$ from the statement of this lemma.
\end{proof}

\subsection{Sensitivity analysis}
\label{appendix_sensitivity_proofs}

\begin{lemma}\label{lem_app_sensi_mean_aux}
Consider count series $\{Y_i(\Delta)\}_{i=1}^K$ associated with a Hawkes process $\mbf H$ defined by intensity function (\ref{hawkes_def}). Suppose that the maximum number of correlated events in the Hawkes process is $B$. Then, the maximal amount of change in the sample mean $\hat\eta$ of the counts is at most $\frac{B}{K}$. 
\end{lemma}

\begin{proof}
For non negative integers $b_{1},\dots, b_{K}$ such that $b_{1}+\dots+b_{K}=B$, we have
\begin{align}
    \left\vert \frac{1}{K}\sum_{i=1}^{K}{Y_{i}}-\frac{1}{K}\sum_{i=1}^{K}({Y_{i}+b_{i}})\right\vert = \frac{B}{K}.
\end{align}

\end{proof}

\begin{lemma} \label{lemma_aux_var_sensitivity}
Consider count series $\{Y_i(\Delta)\}_{i=1}^K$ associated with a Hawkes process $\mbf H$ defined by intensity function (\ref{hawkes_def}). Suppose that the maximum size of groups of correlated events in the Hawkes process is $B$. For non negative integers $y_{1},\dots, y_{K}$ and $b_{1},\dots, b_{K}$ corresponding to $K$ bins, such that $b_{1}+\dots+ b_{K}=B$, and a fixed $\Delta>0$, let the following hold for any  $B$ bins
\begin{align} \label{eq_app_B_bins_discr_cond}
\sum_{i=1}^{B}\left\vert y_{i}-\bar{y}_{K}\right\vert \leq{B^{3/2}}\sqrt{\Delta}\cdot{C_{1}},
\end{align}
where $\bar{y}_{K}$ denotes mean of $y_{i}$s. For the difference between sample variance of $y_{1},\dots, y_{K}$ and $y_{1}+b_{1},\dots, y_{K}+b_{K}$, we have
\begin{align}
     &\left\vert \frac{1}{K-1}\sum_{i=1}^{K}\left(y_{i}+b_{i}-\bar{y}_{K}-\frac{B}{K}\right)^2 - \frac{1}{K-1}\sum_{i=1}^{K} \left(y_{i}-\bar{y}_{K}\right)^2  \right\vert \\
     & \leq \frac{B^{2}}{K} +\frac{2B^{3/2}\sqrt{\Delta}C_{1}}{K-1}
\end{align}
\end{lemma}
\begin{proof}
We have

\begin{align}
    &\left| \frac{1}{K-1}\sum_{i=1}^{K}\left(y_{i}+b_{i}-\bar{y}_{K}-\frac{B}{K}\right)^2 - \frac{1}{K-1}\sum_{i=1}^{K} \left(y_{i}-\bar{y}_{K}\right)^2  \right|\nonumber\\
    &= \left| \frac{1}{K-1}\sum_{i=1}^{K} b_{i}^2 + \frac{1}{K-1}\sum_{i=1}^{K}2 b_i[y_{i}-\bar{y}_{K}]  - \frac{B^2}{K(K-1)} - \frac{B^2}{K^2} \right|\nonumber\\
    & \leq \left\vert \frac{1}{K-1}\sum_{i=1}^{K} b_{i}^2  - \frac{B^2}{K(K-1)} - \frac{B^2}{K^2} \right\vert + \frac{1}{K-1}\sum_{i=1}^{K}2b_{i}\left\vert y_{i}-\bar{y}_{K}\right\vert\\
    & \leq \frac{B^{2}}{K-1}- \frac{B^2}{K(K-1)} - \frac{B^2}{K^2}+ \frac{1}{K-1}\sum_{i=1}^{K}2b_{i}\left\vert y_{i}-\bar{y}_{K}\right\vert\\
    & \leq \frac{B^{2}}{K}+ \frac{1}{K-1}\sum_{i=1}^{K}2b_{i}\left\vert y_{i}-\bar{y}_{K}\right\vert\\
    & \leq \frac{B^{2}}{K}+\frac{2B^{3/2}\sqrt{\Delta}C_{1}}{K-1},
\end{align}
where the last inequality follows from \eqref{eq_app_B_bins_discr_cond}.
\end{proof}

\begin{lemma}\label{lem_app_var_sensitivity_aux}
Consider count series $\{Y_i(\Delta)\}_{i=1}^K$ associated with a Hawkes process $\mbf H$ with parameters $\mu_{\text{lower}}<\mu<\mu_{\text{upper}}$ and $\alpha^{\text{lower}}<\alpha<\alpha^{\text{upper}}$ Suppose that the maximum number of correlated events in the Hawkes process is $B$. Then, with probability at least $1-\gamma$, the maximal amount of change in the sample variance is upper bounded by $\frac{B^2}{K} + \frac{2B^{3/2}\sqrt{\Delta}C_{1}}{(K-1)}$ where $C_{1} = \sqrt{\frac{1.1\cdot\mu_{\text{upper}}}{(1-\alpha_{\text{upper}})^3}\cdot\frac{1}{\gamma}}.$
\end{lemma}

\begin{proof}
By Lemma \ref{lemma_high_prob_bins_event}, with probability at least $1-\gamma$ for any set of $B$ bins we have $\sum_{i=1}^{B}\left\vert Y_{i}-\mathbb{E}(Y_{i})\right\vert < B^{3/2}\sqrt{\Delta}\cdot{C_{1}}$. Thus, as a consequence of Lemma \ref{lemma_aux_var_sensitivity}, with probability at least $1-\gamma$, the maximal amount of change in the sample variance is upper bounded by $\frac{B^2}{K} + \frac{2B^{3/2}\sqrt{\Delta}C_{1}}{(K-1)}$.
\end{proof}

\begin{proof}[Proof of Lemma \ref{lemma_sensitivity}]
Lemma \ref{lemma_sensitivity} combines statements of Lemma \ref{lem_app_sensi_mean_aux} and Lemma \ref{lem_app_var_sensitivity_aux} from this section. See corresponding proofs.
\end{proof}


\begin{corollary}
For $B=C_{2}\log T$, the upper bound in  Lemma \ref{lem_app_var_sensitivity_aux} becomes
\begin{align}
    \frac{C_{2}^{2}(\log T)^2}{K}+\frac{2C_{2}^{3/2}(\log T)^{3/2}\sqrt{\Delta}C_{1}}{(K-1)} = \frac{C_{2}^{2}(\log T)^2 + 2C_{2}^{3/2}C_{1}\cdot\frac{K}{K-1}(\log T)^{3/2}\sqrt{\Delta}}{K}
\end{align}
\end{corollary}

\subsection{Random Differential Privacy}\label{appendix_random_DP}


\begin{proof}[Proof of Lemma \ref{lemma_random_DP_mean_var}]
Recall from Corollary \ref{cor_sensitivity_relation_unaware} that for $0<\gamma\leq 1/2$ and $T\geq \left( {\mu\cdot e^2}/{\gamma} \right)^{5/2}$, the sensitivity of the sample mean and the sample variance are respectively $\frac{C_{2}\log T}{K}$ and $\frac{(C_{2}\log T)^2}{K} + \frac{2(C_{2}\log T)^{3/2}\sqrt{\Delta}C_{1}}{(K-1)}$, with probability at least $1-2\gamma$. Thus, with probability greater than $1-2\gamma$ over the realizations of the generating Hawkes process, $\hat{\eta}_{\text{private}}$ and $\hat{\sigma}^2_{\text{private}}$ are $\epsilon$-differentially private by the properties of Laplace mechanism \cite{dwork2014algorithmic}. 
\end{proof}

\begin{proof}[Proof of Theorem \ref{thm_alpha_mu_random_DP}]
As a consequence of Lemma \ref{lemma_random_DP_mean_var} and the standard composition theorem \cite{kairouz2015composition}, with probability at least $1-2\gamma$, $(\hat{\eta}_{\text{private}},\hat{\sigma}^2_{\text{private}})$ is $2\epsilon$-differentially private. By post processing immunity \cite{dwork2014algorithmic}, it follows that  $\hat{\alpha}_{\text{private}}$ and $\hat{\mu}_{\text{private}}$ are $2\epsilon$-differentially private, with probability at least $1-2\gamma$ over the realizations of the generating Hawkes process. Bearing in mind Definition \ref{def_random_DP}, they are $(2\gamma, 2\epsilon)$-randomly differentially private.
\end{proof}

\subsection{Non private estimators}\label{appendix_non_private_estimators}
\subsubsection{Notion of strongly mixing coefficient}

For a probability space $(\Omega, \mathcal{F}, \mathbb{P})$ and $\mathcal{A}$, $\mathcal{B}$ two sub-sigma algebras of $\mathcal{F}$, the strong mixing coefficient is defined \cite{cheysson2021strong} by
\begin{align}
    \rho(\mathcal{A}, \mathcal{B}) = \sup \{\left\vert \mathbb{P}(A\cap B) - \mathbb{P}(A)\mathbb{P}(B) \right\vert : A\in\mathcal{A}, B\in\mathcal{B}\}.
\end{align}

For a point process $\mbf N$ on $\mathbb{R}$, the strong mixing coefficient takes the form
\begin{align}
    \rho_{\mbf N}(r) := \sup_{t\in\mathbb{R}}\rho(\mathcal{E}_{-\infty}^{t}, \mathcal{E}_{t+r}^{\infty}),
\end{align}
where $\mathcal{E}_{a}^{b}$ stands for sigma-algebra generated by the cylinder sets on $(a,b]$.

For a given sequence $(X_{k})_{k\in\mathbb{Z}}$, strong mixing coefficient is defined as
\begin{align}
    \rho_{\mbf N}(r) := \sup_{n\in\mathbb{Z}}\rho(\mathcal{F}_{-\infty}^{t}, \mathcal{F}_{n+r}^{\infty}),
\end{align}
where $\mathcal{F}_{a}^{b}$ is the sigma algebra generated by $(X_{k})_{a\leq k \leq b}$.

\subsubsection{Strong mixing property of Hawkes process}

In this section we provide a brief overview on the results from \cite{cheysson2021strong} on the strong mixing property of Hawkes process. Their setting studies a general Hawkes process with intensity given by
\begin{align}
   \lambda^{*}_{t} = \mu +\int_{-\infty}^{t}h(t-s)d\mbf N(s) = \mu+\sum_{t_{i}<t} h(t-t_{i}).
\end{align}

Note that Hawkes process (\ref{hawkes_def}) we study in this paper is a special case with exponential impact function i.e. $h(t) = \alpha e^{-\beta t}$ for $\beta=1$.

\begin{lemma}[Theorem 1 from \cite{cheysson2021strong}]\label{lemma_app_strong_mixing_hawkes}
Let $\mathbf{H}(t)$ be a Hawkes process on $\mathbb{R}$ with reproduction function $h(t) = \theta \tilde{h}(t) $, where $\theta=\int_{\mathbb{R}}h(t)<1$ and $\int_{\mathbb{R}}\tilde{h}(t) = 1$. Suppose that there exists $\upsilon>0$ such that the distribution kernel $\tilde{h}$ has a finite moment of order $\upsilon+1$:
\begin{align}\label{eq_app_kernel_moments_condition}
    \int_{\mathbb{R}} t^{1+\upsilon}\tilde{h}(t)dt<\infty.
\end{align}
Then $\mathbf{H}(t)$ is strongly mixing and 
\begin{align}
    \rho_{\mathbf{H}}(r) = O\left(\frac{1}{r^{\upsilon}}\right).
\end{align}
\end{lemma}

\begin{corollary}[Corollary 1 from \cite{cheysson2021strong}]
Under conditions of Lemma \ref{lemma_app_strong_mixing_hawkes}, counting series of Hawkes process $\{Y_{i}(\Delta) = \mbf N(i\Delta)- \mbf N((i-1)\Delta) \}_{i\in\mathbb{Z}}$ where $\mbf N(t)$ is the count process associated with $\mbf H(t)$, is strongly mixing and
\begin{align}
    \rho_{Y}(r) = O\left(\frac{1}{r^{\upsilon}}\right).
\end{align}
\end{corollary}

\begin{corollary}\label{corr_app_strong_mixing_our_hawkes}
For exponential kernel $\tilde{h} = \beta e^{-\beta t}$, condition \eqref{eq_app_kernel_moments_condition} holds for any $\upsilon>0$ as all moments of exponential distribution are finite. Thus, for the strong mixing coefficient $\rho_{Y}(r)$ of counting series of Hawkes process defined by (\ref{hawkes_def}), we have
\begin{align}
    \rho_{Y}(r) = O\left(\frac{1}{r^{\upsilon}}\right),
\end{align}
for any $\upsilon>0$.
\end{corollary}

\subsubsection{Speed of convergence in CLT for weakly dependent random variables}

In this section we state a classical result from \cite{tikhomirov1981convergence} for the speed of convergence in Central Limit Theorem (CLT) for weakly dependent random variables. Let $X_{1}, X_{2}, \dots$ be a sequence of random variables with mean zero and bounded variance. Let
\begin{align}
    F_{n}(z) = \mathbb{P}\left( \frac{1}{\sqrt{\mathbb{E}\left(\sum_{i=1}^{n}X_{i}\right)^2}} \sum_{i=1}^{n}X_{i} <z \right)
\end{align}

Let $\Phi(z)$ denote CDF of standard normal distribution i.e.
\begin{align}
    \Phi(z) = \frac{1}{\sqrt{2\pi}}\int_{-\infty}^{z}e^{-y^2/2}dy.
\end{align}

Let $\rho_{X}(r)$ denote the strong mixing coefficient of the sequence. The following result provides the speed of convergence in the CLT.

\begin{lemma}[Theorem 1 from \cite{tikhomirov1981convergence}]\label{lem_app_tikhomirov}
For a zero mean finite variance sequence $X_{1}, X_{2}, \dots$  let there exist constants $A>0$ and $a>0$ such that the inequalities 
\begin{align}\label{eq_tikhomirov_rho_cond}
    \rho(n)\leq A n^{-a(2+\delta)(1+\delta)/\delta^{2}},
\end{align}
and
\begin{align}\label{eq_app_cond_moment_delta}
    \mathbb{E}\vert X_{1}\vert^{2+\delta}<\infty,
\end{align}
hold for all $r$ and some $0<\delta\leq{1}$. Then
\begin{align}
     \mathbb{E}X_{1}^{2}+2\sum_{k=2}^{\infty}\mathbb{E}(X_{1}X_{k})<\infty,
\end{align}
and if $\mathbb{E}X_{1}^{2}+2\sum_{k=2}^{\infty}\mathbb{E}(X_{1}X_{k})>0$ then there is a constant $A_{1}$ depending just on $A$, $\delta$ and $a$ such that
\begin{align}
    \sup_{z\in\mathbb{R}}\left\vert F_{n}(z)-\Phi(z) \right\vert\leq A_{1}n^{-(\delta/2)(a-1)/(a+1)}.
\end{align}
\end{lemma}

\begin{corollary}\label{cor_app_centered_mean_berry_essen}
For the counting series $\{Y_{i}(\Delta) = \mbf N(i\Delta)- \mbf N((i-1)\Delta) \}_{i\in\mathbb{Z}}$ where $\mbf N(t)$ is the count process associated with Hawkes process $\mbf H(t)$ introduced by (\ref{hawkes_def}), the sequence
\begin{align}
    \left\{\frac{Y_{i}-\eta}{\sigma}\right\}_{i= 1}^K,
\end{align}
where $\eta$ and $\sigma$ are given by \eqref{eq_stationary_mean_hawkes} and \eqref{eq_stationary_var_hawkes} is zero mean unit variance. Third moment of $Y_{i}(\Delta)$ is finite \cite{cui2020elementary} and so it follows that (\ref{eq_app_cond_moment_delta}) holds for $\delta=1$. Moreover, as a consequence of Corollary \ref{corr_app_strong_mixing_our_hawkes} for any $a>0$, we have
\begin{align}
    \rho(K)\leq K^{-6a}.
\end{align}
In particular, condition (\ref{eq_tikhomirov_rho_cond}) is fulfilled for any $a>0$, $\delta=1$ and $A=1$. Thus as a consequence of Lemma \ref{lem_app_tikhomirov}, for any $a>0$ we have
\begin{align}
    \sup_{z\in{\mathbb{R}}}\left\vert\mathbb{P}\left(\frac{\sum_{i=1}^{K}Y_{i}-\eta\cdot K}{\sqrt{K}\sigma}<z\right)-\Phi(z)\right\vert\leq K^{-\frac{1}{2}\cdot\frac{a-1}{a+1}}.
\end{align}
\end{corollary}

\begin{corollary}\label{cor_app_centered_var_berry_essen}
Let $\eta_{4}$ denote $4$th centered moment of count series of Hawkes on the interval of size $\Delta$, i.e. $\eta_{4} = \mathbb{E}\left(Y_{i}(\Delta)-\eta\right)^4$, and let $\kappa = \frac{\eta_{4}}{\sigma^{4}}$. Then
\begin{align}
    \left\{\frac{(Y_{i}-\eta)^{2}-\sigma^{2}}{\sigma^2 \cdot \sqrt{\kappa-1}}\right\}_{i\geq 1},
\end{align}
is zero mean unit variance sequence. Sixth moment of $Y_{i}(\Delta)$ is finite \cite{cui2020elementary} and thus (\ref{eq_app_cond_moment_delta}) holds for $\delta=1$. As a consequence of Corollary \ref{corr_app_strong_mixing_our_hawkes}, for any $a>0$, we have
\begin{align}
    \rho(K)\leq K^{-6a}.
\end{align}
Condition (\ref{eq_tikhomirov_rho_cond}) is fulfilled for any $a>0$, $\delta=1$ and $A=1$. Hence as a consequence of Lemma \ref{lem_app_tikhomirov}, for any $a>0$ we have
\begin{align}
    \sup_{z\in{\mathbb{R}}}\left\vert\mathbb{P}\left(\frac{\sum_{i=1}^{K} (Y_{i}-\eta)^{2}-\sigma^{2}\cdot K}{\sqrt{K}\sigma^2 \cdot \sqrt{\kappa-1}}<z\right)-\Phi(z)\right\vert\leq K^{-\frac{1}{2}\cdot\frac{a-1}{a+1}}.
\end{align}
\end{corollary}

\subsubsection{Sample mean error bound}\label{appendix_mean_berry_essen}

Let $\Psi(\cdot)$ denote the inverse CDF of the standard normal distribution.

\begin{lemma}\label{lemma_mean_berry_essen}
Let $\hat{\eta}$ be the estimate of $\eta=\mathbb{E}(Y_{i}(\Delta))$, the mean value of counting series of Hawkes process started from stationarity. Let $T=K\Delta$ denote the length of time the Hawkes process is observed, divided into $K$ intervals of size $\Delta$. For $\xi>0$, if
\begin{align}
 T> \frac{\sigma^{2}}{\xi^2\Delta}\left(\Psi (1-\delta/4) \right)^{2}.
 \end{align}
 for some $0<\delta\leq 1$, we have
 \begin{align}
    \mathbb{P}\left(\left\vert \frac{\hat{\eta}}{\Delta}-\frac{\mu}{(1-\alpha)} \right\vert>\xi \right)<\delta.
\end{align}
\end{lemma}






\begin{proof}
For a given $K$ and $a>0$, we have
\begin{align}
    \mathbb{P}\left(\left\vert \frac{\hat{\eta}}{\Delta}-\frac{\mu}{(1-\alpha)} \right\vert>\xi \right) &= \mathbb{P}\left( \left\vert \frac{1}{K}\sum_{i=1}^{K} \frac{Y_{i}-\eta}{\Delta} \right\vert>\xi\right)\\
    &=\mathbb{P}\left( \left\vert\frac{\sum_{i=1}^{K} Y_{i}-\eta\cdot K}{\sqrt{K}}\right\vert> \xi\cdot\Delta\cdot\sqrt{K}\right)\\
    &=\mathbb{P}\left( \left\vert\frac{\sum_{i=1}^{K} Y_{i}-\eta\cdot K}{\sqrt{K}\sigma}\right\vert> \frac{\xi\cdot\Delta\cdot\sqrt{K}}{\sigma}\right)\\
    &\leq 2-2\Phi(\frac{\xi\cdot\Delta\cdot\sqrt{K}}{\sigma})+2K^{-\frac{1}{2}\cdot \frac{a-1}{a+1}},
\end{align}
where the inequality follows by Corollary \ref{cor_app_centered_mean_berry_essen}. Note that for any $K\geq 2$ we can find $a>0$ so that $n^{-\frac{1}{2}\cdot\frac{a-1}{a+1}}<\delta/2$. If we have
\begin{align}
    K \geq \frac{\sigma^2}{\xi^2\Delta^2}\left(\Psi\left(1-{\delta}/{4}\right)\right)^2,
\end{align}
then
\begin{align}
    2-2\Phi(\frac{\xi\cdot\Delta\cdot\sqrt{K}}{\sigma})<\delta/2.
\end{align}
Given that $T=K\Delta$, this completes the proof.
\end{proof}

\subsubsection{Sample variance error bound}\label{appendix_var_berry_essen}

\begin{lemma}\label{lem_app_unscaled_var_berry_essen}
Let $\hat{\sigma}^2$ be the estimate of $\sigma^2=\mathbf{Var}(Y_{i}(\Delta))$, the variance of counting series of Hawkes process started from stationarity. Let $T=K\Delta$ denote the length of time the Hawkes process is observed, divided into $K$ intervals of size $\Delta$. For $\xi>0$, if
$$T\geq \max\left\{ \frac{12\Delta\sigma^2}{\delta\xi}, \frac{3\Delta\sigma^2}{\xi}\left( \Psi\left( 1-\frac{\delta}{8}\right) \right)^2, \frac{9\Delta\sigma^2(\eta_4 - \sigma^2)}{\xi^2}\left(\Psi\left( 1-\frac{\delta}{8}\right)\right)^2 \right\}.$$
for some $0<\delta\leq1$, then we have
\begin{align}
    \mathbb{P}\left( \left\vert \hat{\sigma}^{2}-\sigma^{2} \right\vert>\xi \right)<\delta,
\end{align}
\end{lemma}

\begin{proof}
First let us note that
\begin{align}
    \sqrt{K}\left(\hat{\sigma}^2 - \sigma^2\right) &=
    \sqrt{K}\left( \frac{1}{K}\sum_{i=1}^{K}(Y_{i}-\eta)^2-\sigma^2 \right)\\
    &- \sqrt{K}\left( \frac{1}{K}\sum_{i=1}^{K}Y_{i}-\eta \right)\left( \frac{1}{K}\sum_{i=1}^{K}Y_{i}-\eta \right)+\frac{1}{\sqrt{K}}\hat{\sigma}^2.
\end{align}

Thus, for $\kappa=\frac{\eta_{4}}{\sigma^4}$ and $\eta_{4} = \mathbb{E}(Y_{i}(\Delta)-\eta)^4$, we have
\begin{align}
\begin{aligned}
    \frac{\sqrt{K}\left(\hat{\sigma}^2 - \sigma^2\right)}{\sigma^2\sqrt{\kappa-1}} &= \frac{1}{\sqrt{K}}\sum_{i=1}^{K} {\frac{(Y_{i}-\eta)^2-\sigma^2}{\sigma^2\sqrt{\kappa-1}}}
    -\frac{\sqrt{K}}{\sigma^{2}\sqrt{\kappa-1}}\left( \frac{1}{K}\sum_{i=1}^{K}Y_{i}-\eta \right)\left( \frac{1}{K}\sum_{i=1}^{K}Y_{i}-\eta \right)\\
    &+\frac{1}{\sigma^2\sqrt{\kappa-1}}\frac{\hat{\sigma}^2}{\sqrt{K}}.
    \end{aligned}
\end{align}

It follows that
\begin{align}
\begin{aligned}\label{eq_app_var_develop}
    \mathbb{P}\left( \left\vert \hat{\sigma}^2 - \sigma^2 \right\vert >\xi \right) & = \mathbb{P}\left( \frac{\sqrt{K}\left\vert\hat{\sigma}^2 - \sigma^2\right\vert}{\sigma^2\sqrt{\kappa-1}} >\frac{\xi\sqrt{K}}{\sigma^{2}\sqrt{\kappa-1}} \right)\\
    &\leq \mathbb{P}\left( \left\vert\frac{1}{\sqrt{K}}\sum_{i=1}^{K} {\frac{(Y_{i}-\eta)^2-\sigma^2}{\sigma^2\sqrt{\kappa-1}}}\right\vert >\frac{\xi\sqrt{K}}{3\sigma^{2}\sqrt{\kappa-1}}\right)\\
    & + \mathbb{P}\left( \frac{\sqrt{K}}{\sigma^{2}\sqrt{\kappa-1}}\left\vert \frac{1}{K}\sum_{i=1}^{k}Y_{i}-\eta \right\vert^2 >\frac{\xi\sqrt{K}}{3\sigma^{2}\sqrt{\kappa-1}} \right)\\
    & + \mathbb{P}\left( \frac{1}{\sigma^2\sqrt{\kappa-1}}\frac{\hat{\sigma}^2}{\sqrt{K}} >\frac{\xi\sqrt{K}}{3\sigma^{2}\sqrt{\kappa-1}} \right).
    \end{aligned}
\end{align}

Let us consider the three terms on the right hand side of (\ref{eq_app_var_develop}) separately. For the first term, by Corollary \ref{cor_app_centered_var_berry_essen}, for any $a>0$ we have
\begin{align}
\mathbb{P}\left( \left\vert\frac{1}{\sqrt{K}}\sum_{i=1}^{K} {\frac{(Y_{i}-\eta)^2-\sigma^2}{\sigma^2\sqrt{\kappa-1}}}\right\vert >\frac{\xi\sqrt{K}}{3\sigma^{2}\sqrt{\kappa-1}}\right) \leq  2-2\Phi\left(\frac{\xi \sqrt{K}}{3\sigma^2\sqrt{\kappa-1}}\right)+2K^{-\frac{1}{2}\frac{a-1}{a+1}}.
\end{align}
For the second term, for any $a>0$, we have
\begin{align}
\begin{aligned}
\mathbb{P}\left( \frac{\sqrt{K}}{\sigma^{2}\sqrt{\kappa-1}}\left\vert \frac{1}{K}\sum_{i=1}^{K}Y_{i}-\eta \right\vert^2 >\frac{\xi\sqrt{K}}{3\sigma^{2}\sqrt{\kappa-1}} \right) & = \mathbb{P}\left( \left\vert \frac{1}{K}\sum_{i=1}^{K}Y_{i}-\eta \right\vert^2 >\frac{\xi}{3} \right)\\
& = \mathbb{P}\left( \left\vert \frac{1}{K}\sum_{i=1}^{K}(Y_{i}-\eta) \right\vert >\sqrt{\frac{\xi}{3}} \right)\\
& = \mathbb{P}\left( \frac{\sigma}{K} \left\vert \sum_{i=1}^{K}\frac{Y_{i}-\eta}{\sigma} \right\vert >\sqrt{\frac{\xi}{3}} \right)\\
&= \mathbb{P}\left(  \left\vert \frac{1}{\sqrt{K}}\sum_{i=1}^{K}\frac{Y_{i}-\eta}{\sigma} \right\vert >{\frac{\sqrt{\xi}\sqrt{K}}{\sqrt{3}\sigma}} \right)\\
& \leq 2-2\Phi\left(\frac{\sqrt{\xi}\sqrt{K}}{\sqrt{3}\sigma}\right)+2K^{-\frac{1}{2}\frac{a-1}{a+1}},
\end{aligned}
\end{align}
where the last inequality follows by Corollary \ref{cor_app_centered_mean_berry_essen}. Finally for the third term, by Markov inequality we have
\begin{align}
\begin{aligned}
    \mathbb{P}\left( \frac{1}{\sigma^2\sqrt{\kappa-1}}\frac{\hat{\sigma}^2}{\sqrt{K}} >\frac{\xi\sqrt{K}}{3\sigma^{2}\sqrt{\kappa-1}} \right) = \mathbb{P}\left(\hat{\sigma}^2  >\frac{\xi K}{3}\right) \leq \frac{3\mathbb{E}(\hat{\sigma}^2)}{\xi K}=\frac{3\sigma^2}{\xi K},
\end{aligned}
\end{align}
where the last equality holds as $\hat{\sigma}^2$ is unbiased. Thus, it follows from  (\ref{eq_app_var_develop}) that
\begin{align}
    \mathbb{P}\left( \left\vert \hat{\sigma}^2 - \sigma^2 \right\vert >\xi \right) \leq \frac{3\sigma^2}{\xi K }+2-2\Phi\left(\frac{\sqrt{\xi}\sqrt{K}}{\sqrt{3}\sigma}\right)+2-2\Phi\left(\frac{\xi \sqrt{K}}{3\sigma^2\sqrt{\kappa-1}}\right)+4K^{-\frac{1}{2}\frac{a-1}{a+1}}.
\end{align}
Let us note that for any $K\geq{2}$ and any $\delta>0$ we can find $a$ so that $4K^{-\frac{1}{2}\frac{\gamma-1}{\gamma+1}}<\delta/4$.
Moreover, if $K>\frac{12\sigma^2}{\delta\xi}$, we have
\begin{align}
\frac{3\sigma^2}{\xi K}< \delta/4.
\end{align}

Next, if $K\geq \frac{3\sigma^2}{\xi}\left( \Psi\left( 1-\frac{\delta}{8}\right) \right)^2$, we have
\begin{align}
2-2\Phi\left(\frac{\sqrt{\xi}\sqrt{K}}{\sqrt{3}\sigma}\right) <\frac{\delta}{4}.
\end{align}
Similarly, if $$K\geq \frac{9\sigma^2(\eta_4 - \sigma^2)}{\xi^2}\left(\Psi\left( 1-\frac{\delta}{8}\right)\right)^2,$$
then we have
\begin{align}
    2-2\Phi\left(\frac{\xi \sqrt{K}}{3\sigma^2\sqrt{\kappa-1}}\right)<\frac{\delta}{4}.
\end{align}
Thus, if
$$K\geq \max\left\{ \frac{12\sigma^2}{\delta\xi}, \frac{3\sigma^2}{\xi}\left( \Psi\left( 1-\frac{\delta}{8}\right) \right)^2, \frac{9\sigma^2(\eta_4 - \sigma^2)}{\xi^2}\left(\Psi\left( 1-\frac{\delta}{8}\right)\right)^2 \right\}$$
then
$$\mathbb{P}\left( \left\vert \hat{\sigma}^2 - \sigma^2 \right\vert >\xi \right)<\delta.$$
Having in mind $T=K\Delta$, this completes the proof.
\end{proof}

\begin{corollary}\label{lemma_berry_essen_var}
Let $\hat{\sigma}^2$ be the estimate of $\sigma^2=\mathbf{Var}(Y_{i}(\Delta))$, the variance of counting series of Hawkes process started from stationarity. Let $T=K\Delta$ denote the length of time the Hawkes process is observed, divided into $K$ intervals of size $\Delta$. For $\xi>0$, if
$$T\geq \max\left\{ \frac{12\sigma^2}{\delta\xi}, \frac{3\sigma^2}{\xi}\left( \Psi\left( 1-\frac{\delta}{8}\right) \right)^2, \frac{9\sigma^2(\eta_4 - \sigma^2)}{\xi^2\Delta}\left(\Psi\left( 1-\frac{\delta}{8}\right)\right)^2 \right\},$$
for some $0<\delta\leq1$, then we have
\begin{align}
    \mathbb{P}\left( \left\vert \frac{\hat{\sigma}^{2}}{\Delta}-\frac{\sigma^{2}}{\Delta} \right\vert>\xi \right)<\delta.
\end{align}
\end{corollary}

\begin{proof}
This is a straightforward consequence of Lemma \ref{lem_app_unscaled_var_berry_essen}, given that
\begin{align}
    \mathbb{P}\left( \left\vert \frac{\hat{\sigma}^{2}}{\Delta}-\frac{\sigma^{2}}{\Delta} \right\vert>\xi \right) = \mathbb{P}\left( \left\vert \hat{\sigma}^{2}-{\sigma^{2}} \right\vert>\xi\Delta \right).
\end{align}
\end{proof}

\begin{lemma} \label{lemma_var_berry_essen_ratio}
Let $\hat{\sigma}^2$ be the estimate of $\sigma^2=\mathbf{Var}(Y_{i}(\Delta))$, the variance of counting series of Hawkes process started from stationarity. Let $T=K\Delta$ denote the length of time the Hawkes process is observed, divided into $K$ intervals of size $\Delta$ such that
\begin{align}
    \Delta > \frac{4\mu^{\text{upper}}}{(1-\alpha^{\text{upper}})^4}\cdot\frac{1}{\xi},
\end{align}
for some $\xi>0$. If we have
$$T\geq \max\left\{ \frac{12\sigma^2}{\delta\xi}, \frac{3\sigma^2}{\xi}\left( \Psi\left( 1-\frac{\delta}{8}\right) \right)^2, \frac{9\sigma^2(\eta_4 - \sigma^2)}{\xi^2\Delta}\left(\Psi\left( 1-\frac{\delta}{8}\right)\right)^2 \right\},$$
for some $0<\delta\leq1$, then we have
\begin{align}
    \mathbb{P}\left( \left\vert \frac{\hat{\sigma}^{2}}{\Delta}-\frac{\mu}{(1-\alpha)^3} \right\vert>2\xi \right)<\delta.
\end{align}
\end{lemma}

\begin{proof}
Let us recall (\ref{eq_stationary_var_hawkes}), namely
\begin{align}
   \sigma^2 = \frac{\mu\Delta}{(1-\alpha)^3} +\frac{\alpha^2 \mu \left(1-e^{-2(1-\alpha)\Delta}\right)}{2(1-\alpha)^4}-\frac{2\alpha\mu \left(1-e^{-(1-\alpha)\Delta}\right)}{(1-\alpha)^4}.
\end{align}
We have
\begin{align}
\begin{aligned} \label{eq_app_sigma_proxy}
    \left\vert \frac{\sigma^2}{\Delta} - \frac{\mu}{(1-\alpha)^3} \right\vert &= \frac{\mu}{(1-\alpha)^3} \frac{1}{\Delta}\left\vert \frac{\alpha(\alpha-4)}{2(1-\alpha)}+\frac{2\alpha}{1-\alpha}e^{-(1-\alpha)\Delta}-\frac{\alpha^2}{2(1-\alpha)}e^{-2(1-\alpha)\Delta}\right\vert\\
    & \leq \frac{\mu}{(1-\alpha)^3} \frac{1}{\Delta}\left(\frac{\alpha(4-\alpha)}{2(1-\alpha)}+\frac{2\alpha}{(1-\alpha)}+\frac{\alpha^2}{2(1-\alpha)}\right)\\
    &= \frac{\mu}{(1-\alpha)^3} \frac{1}{\Delta}\frac{8\alpha}{2(1-\alpha)}\\
    & \leq \frac{\mu}{(1-\alpha)^3} \frac{1}{\Delta}\frac{4}{(1-\alpha)}\\
    & \leq \frac{1}{\Delta} \frac{4\mu^{\text{upper}}}{(1-\alpha^{\text{upper}})^4}\\
    & \leq \xi
    \end{aligned}
\end{align}
where the last inequality holds due to the assumption $\Delta > \frac{4\mu^{\text{upper}}}{(1-\alpha^{\text{upper}})^4}\cdot\frac{1}{\xi}$ in the statement of the lemma. We have
\begin{align}
    \mathbb{P}\left( \left\vert \frac{\hat{\sigma}^{2}}{\Delta}-\frac{\mu}{(1-\alpha)^3} \right\vert>2\xi \right)\leq \mathbb{P}\left( \left\vert \frac{\hat{\sigma}^{2}}{\Delta}-\frac{\sigma^2}{\Delta} \right\vert>\xi \right) + \mathbb{P}\left( \left\vert \frac{{\sigma}^{2}}{\Delta}-\frac{\mu}{(1-\alpha)^3} \right\vert>\xi \right).
\end{align}
Let us note that (\ref{eq_app_sigma_proxy}) holds with probability $1$. It remains to note that by Corollary \ref{lemma_berry_essen_var} we have $\mathbb{P}\left( \left\vert \frac{\hat{\sigma}^{2}}{\Delta}-\frac{\sigma^2}{\Delta} \right\vert>\xi \right)<\delta$. This completes the proof.
\end{proof}


\subsubsection{Non private estimators error bound: proof of Theorem \ref{thm_non_private_params}}\label{appendix_thm_non_private_params}

\begin{lemma}\label{lemma_app_non_private_params_aux}
Let $\hat\mu$ and $\hat \alpha$ be the estimates of parameters $\mu_{\text{lower}}<\mu<\mu_{\text{upper}}$ and $\alpha^{\text{lower}}<\alpha<\alpha^{\text{upper}}$ of Hawkes process started from stationarity. Let $T=K\Delta$ denote the length of time the Hawkes process is observed, divided into $K$ intervals of size $\Delta$ such that $\Delta > \frac{4\mu^{\text{upper}}}{(1-\alpha^{\text{upper}})^4\xi}$ for some $0<\xi<\frac{\mu_{\text{lower}}}{6}$.
%
Then, if 
\begin{align}
 T\geq \max\Big\{\frac{12\sigma^2}{\delta\xi}, \frac{(\Psi\big (1-\frac{\delta}{4})\big)^{2}\sigma^{2}}{\xi^2\Delta}, 
 \frac{3( \Psi\big( 1-\frac{\delta}{8}\big) )^2\sigma^2}{\xi}, \frac{9\left(\Psi\big( 1-\frac{\delta}{8}\big)\right)^2\sigma^2(\eta_4 - \sigma^2)}{\xi^2\Delta} \Big\},
\end{align}
we have $\mathbb{P}\left( \left\vert \alpha - \hat{\alpha}\right\vert >C_{5}\cdot\xi\right)\leq 2\delta$ and $\mathbb{P}\left(\left\vert \mu-\hat{\mu} \right\vert>C_{6}\cdot\xi\right)\leq 2\delta$, where $C_{5} = \frac{6}{\mu_{\text{lower}}}\cdot\frac{1}{(1-\alpha_{\text{upper}})}$ and $C_{6} = 1+\frac{6\mu_{\text{upper}}}{\mu_{\text{lower}}(1-\alpha_{\text{upper}})^2}+\frac{1}{1-\alpha_{\text{upper}}}$.
\end{lemma}

\begin{proof}
Let us for convenience introduce
\begin{align}
    \omega = \frac{\sigma^2}{\Delta} =\frac{\mu}{(1-\alpha)^3},\qquad
    \hat{\omega} = \frac{\hat{\sigma}^2(\Delta)}{\Delta},
\end{align}
and
\begin{align}
    \vartheta = \frac{\eta}{\Delta} = \frac{\mu}{(1-\alpha)},\qquad
    \hat{ \vartheta} = \frac{\hat{\eta}}{\Delta}.
\end{align}

Before we proceed, let us note that for $\xi\leq\frac{\mu_{\text{lower}}}{6},$ we have
\begin{align}
    \omega - 2\xi = \frac{\mu - 2\xi\cdot(1-\alpha)^3}{(1-\alpha)^3}\geq\frac{\mu-2\xi}{(1-\alpha)^3}\geq\frac{\mu_{\text{lower}}-2\xi}{(1-\alpha)^3}\geq\frac{\mu_{\text{lower}}-2\xi}{(1-\alpha_{\text{lower}})^3}\geq\frac{\mu_{\text{lower}}}{2(1-\alpha_{\text{lower}})^3}
\end{align}
where in the last inequality we use $\xi\leq\frac{\mu_{\text{lower}}}{6}$. Thus, we have
\begin{align} \label{eq_Psi-Xi_bound}
    \frac{1}{\omega - 2\xi}\leq \frac{2(1-\alpha_{\text{lower}})^3}{\mu_{\text{lower}}}\leq \frac{2}{\mu_{\text{lower}}}.
\end{align}
On the event
\begin{align} \label{event_eta_Psi_tight}
    \{ \vert\vartheta-\hat{\vartheta}\vert<\xi\}\cap\{ \vert\omega-\hat{\omega}\vert<2\xi\}
\end{align}

we have
\begin{align}
    \left\vert (1-\alpha)^2 - (1-\hat{\alpha})^2 \right\vert = \left\vert \frac{\vartheta}{\omega} - \frac{\hat{\vartheta}}{\hat{\omega}} \right\vert & = \left\vert \frac{\vartheta(\hat{\omega}-\omega)+\omega(\vartheta-\hat{\vartheta})}{ \omega\hat{\omega}} \right\vert\\
    & \leq 3\cdot\xi\cdot\frac{\vartheta+\omega}{\omega\hat{\omega}}\\
    & \leq 3\cdot\xi\cdot\frac{\vartheta+\omega}{\omega(\Psi-2\xi)}\\
    & \leq 3\cdot\xi\cdot\left(\frac{2}{\omega-2\xi}\right)\\
    &\leq \frac{12\xi}{\mu_{\text{lower}}} ,\label{eq_1-alpha^2_event_bound}
\end{align}
where the third inequality follows from the fact that $\vartheta = \frac{\mu}{(1-\alpha)}$ and therefore $\frac{\vartheta}{\omega} = (1-\alpha)^2<1$. The last inequality is a consequence of (\ref{eq_Psi-Xi_bound}). Let us further note that
\begin{align}
    \left\vert \alpha-\hat{\alpha} \right\vert = \left\vert \sqrt{(1-\alpha)^2} - \sqrt{(1-\hat{\alpha})^2} \right\vert \leq \frac{1}{2(1-\alpha_{\text{upper}})}\cdot \left\vert (1-\alpha)^2 - (1-\hat{\alpha})^2 \right\vert 
\end{align}
and thus on the event (\ref{event_eta_Psi_tight}), by (\ref{eq_1-alpha^2_event_bound}) we have 
\begin{align}
    \left\vert \alpha-\hat{\alpha} \right\vert \leq \frac{6\xi}{\mu_{\text{lower}}}\cdot\frac{1}{(1-\alpha_{\text{upper}})}=C_{5}\cdot\xi.
\end{align}

Next, on the event (\ref{event_eta_Psi_tight}) we have
\begin{align}
    \left\vert \hat{\mu}-\mu \right\vert = \left\vert \vartheta(1-\alpha) - \hat{\vartheta}(1-\hat{\alpha}) \right\vert &= \left\vert \vartheta(1-\alpha) - \hat{\vartheta}(1-\alpha) +\hat{\vartheta}(1-\alpha)- \hat{\vartheta}(1-\hat{\alpha}) \right\vert\\
    & \leq (1-\alpha)\vert\hat{\vartheta}-\vartheta\vert +\hat{\vartheta}\vert\alpha-\hat{\alpha}\vert\\
    & \leq \xi + \hat{\vartheta}\vert\alpha-\hat{\alpha}\vert\\
    & \leq \xi + (\vartheta+\xi)\vert\alpha-\hat{\alpha}\vert\\
    & \leq \xi + (\vartheta+\xi)\cdot C_{5}\cdot\xi\\
    & \leq \xi + \left(\frac{\mu}{(1-\alpha)}+\xi\right)\cdot C_{5}\cdot\xi\\
    & \leq \xi+\left(\frac{\mu_{\text{upper}}}{1-\alpha_{\text{upper}}}+\frac{\mu_{\text{lower}}}{6}\right)\cdot C_{5}\cdot\xi\\
    & = C_{6}\cdot\xi
\end{align}
It remains to note that the probability of event (\ref{event_eta_Psi_tight}) is at least $1-2\delta$ as the consequence of Lemma \ref{lemma_mean_berry_essen} and Lemma \ref{lemma_var_berry_essen_ratio}, and a union bound argument.
\end{proof}

Finally, we prove Theorem \ref{thm_non_private_params}.

\begin{proof}[Proof of Theorem \ref{thm_non_private_params}]
For $0<\xi<\frac{C_{9}\mu_{\text{lower}}}{6}$, let
$$\tilde{\xi} = \frac{1}{C_{9}}\xi.$$ As a consequence of Lemma \ref{lemma_app_non_private_params_aux}, for $T$ as in the statement of the theorem, we have
\begin{align}
    \mathbb{P}\left( \left\vert \alpha - \hat{\alpha}\right\vert >C_{5}\cdot\tilde{\xi}\right)&\leq \delta, \label{eq_aux_alpha_C5}\\
    \mathbb{P}\left( \left\vert \mu - \hat{\mu}\right\vert >C_{6}\cdot\tilde{\xi}\right)&\leq \delta \label{eq_aux_mu_C6}
\end{align}
where we recall $C_{5} = \frac{6}{\mu_{\text{lower}}}\cdot\frac{1}{(1-\alpha_{\text{upper}})}$ and $C_{6} = 1+\frac{6\mu_{\text{upper}}}{\mu_{\text{lower}}(1-\alpha_{\text{upper}})^2}+\frac{1}{1-\alpha_{\text{upper}}}$. It remains to note that $C_{5}<C_{9}$ and $C_{6}<C_{9}$ and thus
\begin{align}
     \mathbb{P}\left( \left\vert \alpha - \hat{\alpha}\right\vert >\xi\right) < \mathbb{P}\left( \left\vert \alpha - \hat{\alpha}\right\vert >\frac{C_{5}}{C_{9}}\xi\right) = \mathbb{P}\left( \left\vert \alpha - \hat{\alpha}\right\vert >C_{5}\cdot\tilde{\xi}\right)\leq \delta,
\end{align}
where the last inequality follows by (\ref{eq_aux_alpha_C5}). Analogously, by (\ref{eq_aux_mu_C6}) we have
\begin{align}
     \mathbb{P}\left( \left\vert \mu - \hat{\mu}\right\vert >\xi\right) < \mathbb{P}\left( \left\vert \mu - \hat{\mu}\right\vert >\frac{C_{6}}{C_{9}}\xi\right) = \mathbb{P}\left( \left\vert \mu - \hat{\mu}\right\vert >C_{6}\cdot\tilde{\xi}\right)\leq \delta.
\end{align}
This completes the proof.
\end{proof}

\subsection{Private estimates}\label{appendix_private_estimates}
\subsubsection{Private sample mean error bound}\label{appendix_private_sample_mean}

\begin{lemma} \label{lemma_priv_complex_mean}
Let ${\hat{\eta}_{\text{private}}}$ be the private estimate as per \eqref{eq_DP_mean} of $\eta=\mathbb{E}(Y_{i}(\Delta))$, the mean value of counting series of Hawkes process started from stationarity. Let $T=K\Delta$ denote the length of time the Hawkes process is observed, divided into $K$ intervals of size $\Delta$. For $\xi>0$, if
\begin{align}
 T> \frac{\sigma^{2}}{\xi^2\Delta}\left(\Psi (1-\delta/4) \right)^{2}.
 \end{align}
 for some $0<\delta\leq 1$,
for $(2\gamma,\epsilon)$-DP estimate $\eta_{\text{private}}$, we have
\begin{align}
    \mathbb{P}\left( \left\vert \frac{\hat{\eta}_{\text{private}}}{\Delta}-\frac{\mu}{(1-\alpha)} \right\vert \geq 2\xi\right)\leq \exp\left(-\frac{\epsilon\cdot\xi}{C_{2}}\cdot\frac{T}{\log T}\right)+\delta,
\end{align}
where $C_{2} = \frac{3}{(1-\alpha_{\text{upper}})^2}$.
\end{lemma}

\begin{proof}
For $\xi>0$, bearing in mind \eqref{eq_DP_mean}, we have
\begin{align}
    \mathbb{P}\left( \left\vert \frac{\hat{\eta}_{\text{private}}}{\Delta}-\frac{\hat{\eta}}{\Delta} \right\vert \geq \xi\right)\leq \exp\left(-\frac{\epsilon\cdot\xi}{C_{2}}\cdot\frac{T}{\log T}\right)
\end{align}
by the tail behaviour of Laplace distribution (see e.g. Theorem 3.8. of \cite{dwork2014algorithmic}). As a consequence of Lemma \ref{lemma_mean_berry_essen}, we have
\begin{align}
    \mathbb{P}\left(\left\vert \frac{\hat{\eta}}{\Delta}-\frac{\mu}{(1-\alpha)} \right\vert>\xi \right)<\delta.
\end{align}
It remains to note
\begin{align}
    \mathbb{P}\left( \left\vert \frac{\hat{\eta}_{\text{private}}}{\Delta}-\frac{\hat{\eta}}{\Delta} \right\vert \geq \xi\right)\leq  \mathbb{P}\left( \left\vert \frac{\hat{\eta}_{\text{private}}}{\Delta}-\frac{\hat{\eta}}{\Delta} \right\vert \geq \xi\right)+  \mathbb{P}\left(\left\vert \frac{\hat{\eta}}{\Delta}-\frac{\mu}{(1-\alpha)} \right\vert>\xi \right).
\end{align}
This completes the proof.
\end{proof}

\subsubsection{Private sample variance error bound }\label{appendix_private_sample_var}

\begin{lemma} \label{lemma_priv_var_accurate}
Let $\hat{\sigma}^{2}_{\text{private}}$ be the private estimate as per \eqref{eq_DP_var} of $\sigma^2=\mathbf{Var}(Y_{i}(\Delta))$  the variance of counting series of Hawkes process started from stationarity. Let $T=K\Delta$ denote the length of time the Hawkes process is observed, divided into $K$ intervals of size $\Delta$. In addition, let $C_{1} = \sqrt{\frac{1.1\cdot\mu_{\text{upper}}}{(1-\alpha_{\text{upper}})^3}\cdot\frac{1}{\gamma}}$ and $C_{2} = \frac{3}{(1-\alpha_{\text{upper}})^2}$. For the choice of $\Delta =c \log T$ for some constant $c$, if
\begin{align}
    &T\geq \max\left\{ \frac{12\sigma^2}{\delta\xi}, \frac{3\sigma^2}{\xi}\left( \Psi\left( 1-\frac{\delta}{8}\right) \right)^2, \frac{9\sigma^2(\eta_4 - \sigma^2)}{\xi^2 \Delta}\left(\Psi\left( 1-\frac{\delta}{8}\right)\right)^2 \right\},\\
    &\frac{T}{(\log T)^{5/2}} \geq \frac{4\sqrt{c}C_{1}C_{2}^2}{\epsilon\xi}\log \left(\frac{1}{\delta}\right),
\end{align}

for some for some $\xi>0$ and $0<\delta\leq1$, then for $(2\gamma,\epsilon)$-DP estimate $\hat{\sigma}^{2}_{\text{private}}$, we have
\begin{align}
   \mathbb{P}\left( \left\vert \frac{\hat{\sigma}^{2}_{\text{private}}}{\Delta} - \frac{\sigma^2}{\Delta} \right\vert \geq 2\xi \right)\leq 2\delta.
\end{align}
\end{lemma}

\begin{proof}
We have 
\begin{align}
   \mathbb{P}\left( \left\vert \frac{\hat{\sigma}^{2}_{\text{private}}}{\Delta} - \frac{\sigma^2}{\Delta} \right\vert \geq 2\xi \right)\leq \mathbb{P}\left( \left\vert \frac{\hat{\sigma}^{2}_{\text{private}}}{\Delta} - \frac{\hat{\sigma}^2}{\Delta} \right\vert \geq \xi \right) + \mathbb{P}\left( \left\vert \frac{\hat{\sigma}^{2}}{\Delta} - \frac{{\sigma}^2}{\Delta} \right\vert \geq \xi \right).
\end{align}
As a consequence of Corollary \ref{lemma_berry_essen_var}, we have  $\mathbb{P}\left( \left\vert \frac{\hat{\sigma}^{2}}{\Delta} - \frac{{\sigma}^2}{\Delta} \right\vert \geq \xi \right)\leq \delta$. Therefore, it suffices to show that also $\mathbb{P}\left( \left\vert \frac{\hat{\sigma}^{2}_{\text{private}}}{\Delta} - \frac{\hat{\sigma}^2}{\Delta} \right\vert \geq \xi \right)\leq \delta$. Bearing in mind \eqref{eq_DP_var}, when choosing $\Delta=c\log T$, for the tail of Laplace distribution we have
\begin{align}
  \mathbb{P}\left( \left\vert \frac{\hat{\sigma}^{2}_{\text{private}}}{\Delta} - \frac{\hat{\sigma}^2}{\Delta} \right\vert \geq \xi \right) &= \exp\left(-\frac{t\cdot\epsilon\cdot\xi}{C_{2}^{2}(\log T)^2 + 2C_{2}^{3/2}C_{1}\cdot\frac{K}{K-1}(\log T)^{3/2}\sqrt{\Delta}}\right)\\
  & = \exp\left(-\frac{t\cdot\epsilon\cdot\xi}{C_{2}^{2}(\log T)^2 + 2C_{2}^{3/2}C_{1}\cdot\frac{K}{K-1}(\log T)^{3/2}\sqrt{c\log T}}\right)\\
  & \leq \exp\left(-\frac{t\cdot\epsilon\cdot\xi}{4C_{1}C_{2}^{2}\sqrt{c}(\log T)^{5/2}}\right)\\
  & \leq \delta,
\end{align}
where in the last inequality we rely on assumption $\frac{T}{(\log T)^{5/2}} \geq \frac{4\sqrt{c}C_{1}C_{2}^2}{\epsilon\xi}\log \left(\frac{1}{\delta}\right)$. This completes the proof.
\end{proof}

\begin{lemma} \label{lemma_priv_complex_var}
Let $\hat{\sigma}^{2}_{\text{private}}$ be the private estimate as per \eqref{eq_DP_var} of $\sigma^2=\mathbf{Var}(Y_{i}(\Delta))$  the variance of counting series of Hawkes process started from stationarity. Let $T=K\Delta$ denote the length of time the Hawkes process is observed, divided into $K$ intervals of size $\Delta$
such that
\begin{align}
    \Delta > \frac{4\mu^{\text{upper}}}{(1-\alpha^{\text{upper}})^4}\cdot\frac{1}{\xi},
\end{align}
for some $\xi>0$. In addition, let $C_{1} = \sqrt{\frac{1.1\cdot\mu_{\text{upper}}}{(1-\alpha_{\text{upper}})^3}\cdot\frac{1}{\gamma}}$ and $C_{2} = \frac{3}{(1-\alpha_{\text{upper}})^2}$. For the choice of $\Delta =c \log T$ for some constant $c$, if
\begin{align}
    &T\geq \max\left\{ \frac{12\sigma^2}{\delta\xi}, \frac{3\sigma^2}{\xi}\left( \Psi\left( 1-\frac{\delta}{8}\right) \right)^2, \frac{9\sigma^2(\eta_4 - \sigma^2)}{\xi^2\Delta}\left(\Psi\left( 1-\frac{\delta}{8}\right)\right)^2 \right\},\\
    &\frac{T}{(\log T)^{5/2}} \geq \frac{4\sqrt{c}C_{1}C_{2}^2}{\epsilon\xi}\log \left(\frac{1}{\delta}\right),
\end{align}
for some $0<\delta\leq1$, then for $(2\gamma,\epsilon)$-DP estimate $\hat{\sigma}^{2}_{\text{private}}$, we have
\begin{align}
   \mathbb{P}\left( \left\vert \frac{\hat{\sigma}^{2}_{\text{private}}}{\Delta} - \frac{\mu}{(1-\alpha)^3} \right\vert \geq 3\xi \right)\leq 2\delta.
\end{align}
\end{lemma}

\begin{proof}
We have
\begin{align}
    \mathbb{P}\left( \left\vert \frac{\hat{\sigma}^{2}_{\text{private}}}{\Delta} - \frac{\mu}{(1-\alpha)^3} \right\vert \geq 3\xi \right) = \mathbb{P}\left( \left\vert \frac{\hat{\sigma}^{2}_{\text{private}}}{\Delta} - \frac{\sigma^2}{\Delta} \right\vert \geq 2\xi \right) + \mathbb{P}\left( \left\vert \frac{\sigma^2}{\Delta} -\frac{\mu}{(1-\alpha)^3}\right\vert \geq \xi \right)
\end{align}
The first term on the right hand side is not larger than $2\delta$ by Lemma \ref{lemma_priv_var_accurate}. The second term is zero as \eqref{eq_app_sigma_proxy} holds with probability 1 for $\Delta > \frac{4\mu^{\text{upper}}}{(1-\alpha^{\text{upper}})^4}\cdot\frac{1}{\xi}$. This completes the proof.
\end{proof}

\subsubsection{Private estimators error bound: proof of Theorem \ref{thm_private_params}} \label{appendix_thm_private_params}

\begin{lemma} \label{lemma_app_private_alpha_mu}
Let $\hat{\mu}_{\text{private}}$ and $\hat{\alpha}_{\text{private}}$ be the estimates of parameters $\mu_{\text{lower}}<\mu<\mu_{\text{upper}}$ and $\alpha^{\text{lower}}<\alpha<\alpha^{\text{upper}}$ of Hawkes process started from stationarity. Let $T=K\Delta$ denote the length of time the Hawkes process is observed, divided into $K$ intervals of size $\Delta$
such that
\begin{align}
    \Delta > \frac{4\mu^{\text{upper}}}{(1-\alpha^{\text{upper}})^4}\cdot\frac{1}{\xi},
\end{align}
for some $0<\xi<\frac{\mu_{\text{lower}}}{6}$. For the choice of $\Delta =c \log T$ for some constant $c$, if
\begin{align}
    &T\geq \max\left\{ \frac{\sigma^{2}}{\xi^2\Delta}\Psi \left(1-\frac{\delta}{4}\right)^{2}, \frac{12\sigma^2}{\delta\xi}, \frac{3\sigma^2}{\xi}\Psi\left( 1-\frac{\delta}{8}\right)^2, \frac{9\sigma^2(\eta_4 - \sigma^2)}{\xi^2\Delta}\Psi\left( 1-\frac{\delta}{8}\right)^2 \right\},\\
    &\frac{T}{(\log T)^{5/2}} \geq \frac{4\sqrt{c}C_{1}C_{2}^2}{\epsilon\xi}\log \left(\frac{1}{\delta}\right) \label{cond_app_T_log_T_0},
\end{align}
for some $0<\delta\leq1$, then for $(2\gamma,2\epsilon)$-DP estimates we have $\mathbb{P}\left(\left\vert\hat{\alpha}_{\text{private}}-\alpha\right\vert>C_{7}\xi\right)\leq 4\delta$ and $\mathbb{P}\left(\left\vert\hat{\mu}_{\text{private}}-\mu\right\vert>C_{8}\xi\right)\leq 4\delta$, where $C_{7}=\frac{10}{\mu_{\text{lower}}(1-\alpha_{\text{upper}})}$ and $C_{8} = 2+\frac{10\mu_{\text{upper}}}{\mu_{\text{lower}}(1-\alpha_{\text{upper}})^2}+\frac{10}{3(1-\alpha_{\text{upper}})}.$
\end{lemma}

\begin{proof}
Condition \eqref{cond_app_T_log_T_0} implies $\exp\left(-\frac{\epsilon\cdot\xi}{C_{2}}\cdot\frac{T}{\log T}\right)<\delta$, and thus by Lemma \ref{lemma_priv_complex_mean}, we have $\mathbb{P}\left( \left\vert \frac{\hat{\eta}_{\text{private}}}{\Delta}-\frac{\mu}{(1-\alpha)} \right\vert \geq 2\xi\right)\leq 2\delta$. Also, following Lemma \ref{lemma_priv_complex_var} we have $\mathbb{P}\left( \left\vert \frac{\hat{\sigma}^{2}_{\text{private}}}{\Delta} - \frac{\mu}{(1-\alpha)^3} \right\vert \geq 3\xi \right)\leq 2\delta.$ The proof closely resembles that of Lemma \ref{lemma_app_non_private_params_aux}. Let us for convenience introduce
\begin{align}
    \omega = \frac{\sigma^2}{\Delta} =\frac{\mu}{(1-\alpha)^3},\qquad
    \hat{\omega} = \frac{\hat{\sigma}^2(\Delta)}{\Delta},
\end{align}
and
\begin{align}
    \vartheta = \frac{\eta}{\Delta} = \frac{\mu}{(1-\alpha)},\qquad
    \hat{ \vartheta} = \frac{\hat{\eta}}{\Delta}.
\end{align}

Let us note that for $\xi\leq\frac{\mu_{\text{lower}}}{6},$ we have
\begin{align}
    \omega - 3\xi = \frac{\mu - 3\xi\cdot(1-\alpha)^3}{(1-\alpha)^3}\geq\frac{\mu-2\xi}{(1-\alpha)^3}\geq\frac{\mu_{\text{lower}}-3\xi}{(1-\alpha)^3}\geq\frac{\mu_{\text{lower}}-3\xi}{(1-\alpha_{\text{lower}})^3}\geq\frac{\mu_{\text{lower}}}{2(1-\alpha_{\text{lower}})^3}
\end{align}
where in the last inequality we use $\xi\leq\frac{\mu_{\text{lower}}}{6}$. Thus, we have
\begin{align} \label{eq_Psi-Xi_bound_2nd}
    \frac{1}{\omega - 3\xi}\leq \frac{2(1-\alpha_{\text{lower}})^3}{\mu_{\text{lower}}}\leq \frac{2}{\mu_{\text{lower}}}.
\end{align}
On the event
\begin{align} \label{event_eta_Psi_tight_2nd}
    \{ \vert\vartheta-\hat{\vartheta}\vert<2\xi\}\cap\{ \vert\omega-\hat{\omega}\vert<3\xi\}
\end{align}

we have
\begin{align}
    \left\vert (1-\alpha)^2 - (1-\hat{\alpha})^2 \right\vert = \left\vert \frac{\vartheta}{\omega} - \frac{\hat{\vartheta}}{\hat{\omega}} \right\vert & = \left\vert \frac{\vartheta(\hat{\omega}-\omega)+\omega(\vartheta-\hat{\vartheta})}{ \omega\hat{\omega}} \right\vert\\
    & \leq 5\cdot\xi\cdot\frac{\vartheta+\omega}{\omega\hat{\omega}}\\
    & \leq 5\cdot\xi\cdot\frac{\vartheta+\omega}{\omega(\Psi-3\xi)}\\
    & \leq 5\cdot\xi\cdot\left(\frac{2}{\omega-3\xi}\right)\\
    &\leq \frac{20\xi}{\mu_{\text{lower}}} \label{eq_1-alpha^2_event_bound_2nd}
\end{align}
where the third inequality follows from $\vartheta = \frac{\mu}{(1-\alpha)}$, implying  $\frac{\vartheta}{\omega} = (1-\alpha)^2<1$. The last inequality follows from  (\ref{eq_Psi-Xi_bound_2nd}). Let us further note that
\begin{align}
    \left\vert \alpha-\hat{\alpha} \right\vert = \left\vert \sqrt{(1-\alpha)^2} - \sqrt{(1-\hat{\alpha})^2} \right\vert \leq \frac{1}{2(1-\alpha_{\text{upper}})}\cdot \left\vert (1-\alpha)^2 - (1-\hat{\alpha})^2 \right\vert 
\end{align}
and thus on the event (\ref{event_eta_Psi_tight_2nd}), by (\ref{eq_1-alpha^2_event_bound_2nd}) we have 
\begin{align}
    \left\vert \alpha-\hat{\alpha} \right\vert \leq \frac{10\xi}{\mu_{\text{lower}}}\cdot\frac{1}{(1-\alpha_{\text{upper}})}=C_{7}\cdot\xi.
\end{align}

Next, on the event (\ref{event_eta_Psi_tight_2nd}) we have
\begin{align}
    \left\vert \hat{\mu}-\mu \right\vert = \left\vert \vartheta(1-\alpha) - \hat{\vartheta}(1-\hat{\alpha}) \right\vert &= \left\vert \vartheta(1-\alpha) - \hat{\vartheta}(1-\alpha) +\hat{\vartheta}(1-\alpha)- \hat{\vartheta}(1-\hat{\alpha}) \right\vert\\
    & \leq (1-\alpha)\vert\hat{\vartheta}-\vartheta\vert +\hat{\vartheta}\vert\alpha-\hat{\alpha}\vert\\
    & \leq \xi + \hat{\vartheta}\vert\alpha-\hat{\alpha}\vert\\
    & \leq 2\xi + (\vartheta+2\xi)\vert\alpha-\hat{\alpha}\vert\\
    & \leq 2\xi + (\vartheta+2\xi)\cdot C_{7}\cdot\xi\\
    & \leq 2\xi + \left(\frac{\mu}{(1-\alpha)}+2\xi\right)\cdot C_{7}\cdot\xi\\
    & \leq 2\xi+\left(\frac{\mu_{\text{upper}}}{1-\alpha_{\text{upper}}}+\frac{2\mu_{\text{lower}}}{6}\right)\cdot C_{7}\cdot\xi\\
    & = C_{8}\cdot\xi
\end{align}
It remains to note that the probability of event (\ref{event_eta_Psi_tight_2nd}) is at least $1-4\delta$ by the union bound.
\end{proof}

Finally, we are ready to prove Theorem \ref{thm_private_params}.

\begin{proof}[Proof of Theorem \ref{thm_private_params}]
For $0<\xi<\frac{C_{9}\mu_{lower}}{6}$, let
$$\tilde{\xi} = \frac{1}{C_{9}}\xi,$$ 
where $C_{9}=\max\{ \frac{10}{\mu_{\text{lower}}(1-\alpha_{\text{upper}})}, 2+\frac{10\mu_{\text{upper}}}{\mu_{\text{lower}}(1-\alpha_{\text{upper}})^2}+\frac{10}{3(1-\alpha_{\text{upper}})} \}.$ As a consequence of Lemma \ref{lemma_app_private_alpha_mu}, if 
\begin{align}
    &T\geq \max\left\{ \frac{C_{9}^2\sigma^{2}}{\xi^2\Delta}\Psi (1-\frac{\delta}{16})^{2}, \frac{48C_{9}\sigma^2}{\delta\xi}, \frac{3C_{9}\sigma^2}{\xi}\Psi( 1-\frac{\delta}{32})^2, \frac{9C_{9}^2\sigma^2(\eta_4 - \sigma^2)}{\xi^2\Delta}\Psi( 1-\frac{\delta}{32})^2 \right\}\label{cond_app_dp_1_aux}
\end{align}
and
\begin{align}
    &\frac{T}{(\log T)^{5/2}} \geq \frac{4\sqrt{c}C_{1}C_{2}^2C_{9}}{\epsilon\xi}\log \left(\frac{4}{\delta}\right) \label{cond_app_T_log_T}
\end{align}
for some $0<\delta\leq1$, then we have $\mathbb{P}\left(\left\vert\hat{\alpha}_{\text{private}}-\alpha\right\vert>C_{7}\tilde{\xi}\right)\leq \delta$ and $\mathbb{P}\left(\left\vert\hat{\mu}_{\text{private}}-\mu\right\vert>C_{8}\tilde{\xi}\right)\leq \delta$, where $C_{7}=\frac{10}{\mu_{\text{lower}}(1-\alpha_{\text{upper}})}$ and $C_{8} = 2+\frac{10\mu_{\text{upper}}}{\mu_{\text{lower}}(1-\alpha_{\text{upper}})^2}+\frac{10}{3(1-\alpha_{\text{upper}})}.$ Given that $C_{7}\leq C_{9}$ and $C_{8}\leq C_{9}$, we have
\begin{align}
    \mathbb{P}\left( \left\vert\hat{\alpha}_{\text{private}}-\alpha\right\vert>{\xi} \right) < \mathbb{P}\left( \left\vert\hat{\alpha}_{\text{private}}-\alpha\right\vert>\frac{C_{7}}{C_{9}}{\xi} \right) = \mathbb{P}\left( \left\vert\hat{\alpha}_{\text{private}}-\alpha\right\vert>{C_{7}}\tilde{\xi} \right)\leq \delta.
\end{align}
Similarly, we obtain $\mathbb{P}\left( \left\vert\hat{\mu}_{\text{private}}-\mu\right\vert>{\xi} \right)\leq \delta$. Finally, it follows from equation \eqref{eq_stationary_var_hawkes} that
$$\sigma^2 < \frac{\mu_{\text{upper}\Delta}}{(1-\alpha_{\text{upper}})^3},$$ and therefore conditions \eqref{condition_dp_1} and \eqref{condition_dp_2} from the statement of the theorem are sufficient for \eqref{cond_app_dp_1_aux} to hold. This completes the proof.
\end{proof}


\subsection{Experiments}

Throughout the experiment section we define the utility as the normalized absolute error. Given an estimate $(\hat \mu, \hat \alpha)$ of the Hawkes process parameters $(\mu, \alpha)$, the normalized absolute error is defined as:

$$
E_{\mu} = \frac{|\hat \mu - \mu |}{\mu} \quad \quad {\rm and} \quad \quad E_{\alpha} = \frac{|\hat \alpha - \alpha |}{\alpha}
$$

Figure~\ref{fig:pu_synth_baselines} provides the privacy-utility tradeoff for the estimation of $\mu_1$ (left) and $\mu_2$ (right). This figure is the same as Figure~\ref{fig:pu_synth}, but the baseline intensity parameters instead. The same takeaways provided in Section~\ref{sect_experiments} apply.

\begin{figure}[!h]
    \centering
    \includegraphics[width=1.0\textwidth]{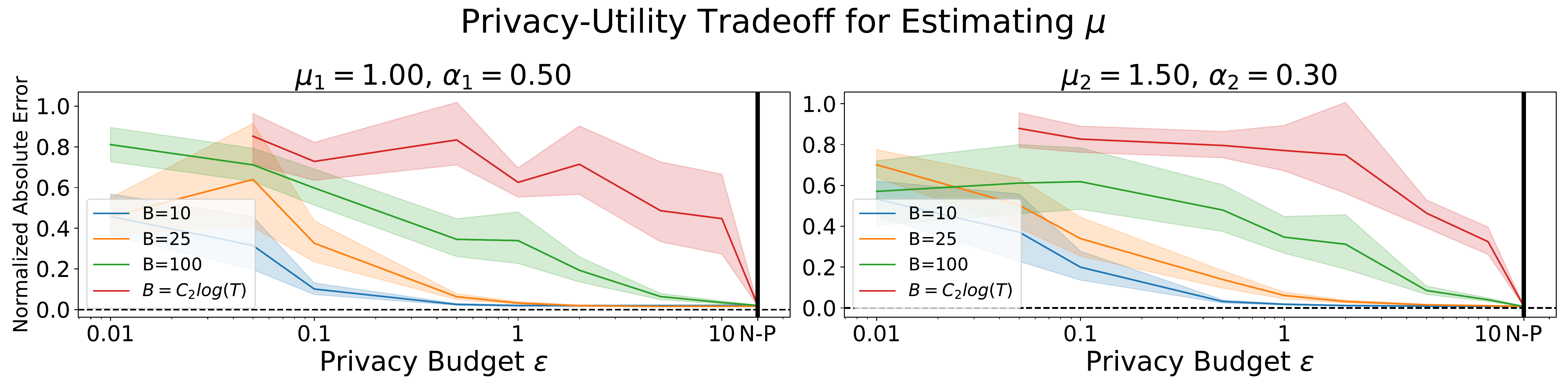}
    \caption{Privacy-utility trade-off for our differentially private estimator for $\mu_1$ (left) and $\mu_2$ (right). Mean and $95\%$ bands observed over 50 repetitions at each privacy level, along with the non-private estimation error appended at the rightmost part of the x-axis. Utility increases with a larger privacy budget, and decreases with the maximum tree length $B$. Same takeaways provided in Section~\ref{sect_experiments} for Figure~\ref{fig:pu_synth} apply.  }
    \label{fig:pu_synth_baselines}
\end{figure}

All experiments were run on a single machine with a 2 GHz Quad-Core Intel Core i5 CPU and 32GB of RAM. Finally, the licenses attached to the real datasets used all allow free use of the data, including commercial use. The \textit{MathOverflow} data are a part of the ``Stack Exchange Data Dump'', with a Creative Commons license BY-SA 4.0\footnote{\url{https://creativecommons.org/licenses/by-sa/4.0/}}. The \textit{911 Calls} dataset has an open database contents license \footnote{\url{https://opendatacommons.org/licenses/dbcl/1-0/}}, while the \textit{MOOC} dataset is equipped with a Creative Commons CC0 1.0 license \footnote{\url{https://creativecommons.org/publicdomain/zero/1.0/}}.


\begin{thebibliography}{10}

\bibitem{abadi2016deep}
M.~Abadi, A.~Chu, I.~Goodfellow, H.~B. McMahan, I.~Mironov, K.~Talwar, and
  L.~Zhang.
\newblock Deep learning with differential privacy.
\newblock In {\em Proceedings of the 2016 ACM SIGSAC conference on computer and
  communications security}, pages 308--318, 2016.

\bibitem{athreya2004branching}
K.~B. Athreya, P.~E. Ney, and P.~Ney.
\newblock {\em Branching processes}.
\newblock Courier Corporation, 2004.

\bibitem{bacry2017tick}
E.~{Bacry}, M.~{Bompaire}, S.~{Ga{\"i}ffas}, and S.~{Poulsen}.
\newblock {tick: a Python library for statistical learning, with a particular
  emphasis on time-dependent modeling}.
\newblock {\em ArXiv e-prints}, July 2017.

\bibitem{bacry2015hawkes}
E.~Bacry, I.~Mastromatteo, and J.-F. Muzy.
\newblock Hawkes processes in finance.
\newblock {\em Market Microstructure and Liquidity}, 1(01):1550005, 2015.

\bibitem{bentkus1997berry}
V.~Bentkus, F.~G{\"o}tze, and A.~Tikhomirov.
\newblock Berry-esseen bounds for statistics of weakly dependent samples.
\newblock {\em Bernoulli}, pages 329--349, 1997.

\bibitem{blair1976rational}
J.~Blair, C.~Edwards, and J.~H. Johnson.
\newblock Rational chebyshev approximations for the inverse of the error
  function.
\newblock {\em Mathematics of computation}, 30(136):827--830, 1976.

\bibitem{bordenave2007large}
C.~Bordenave and G.~L. Torrisi.
\newblock Large deviations of poisson cluster processes.
\newblock {\em Stochastic Models}, 23(4):593--625, 2007.

\bibitem{cao2018quantifying}
Y.~Cao, M.~Yoshikawa, Y.~Xiao, and L.~Xiong.
\newblock Quantifying differential privacy in continuous data release under
  temporal correlations.
\newblock {\em IEEE transactions on knowledge and data engineering},
  31(7):1281--1295, 2018.

\bibitem{cheysson2021strong}
F.~Cheysson and G.~Lang.
\newblock Strong-mixing rates for hawkes processes and application to whittle
  estimation from count data.
\newblock 2021.

\bibitem{cho2020contact}
H.~Cho, D.~Ippolito, and Y.~W. Yu.
\newblock Contact tracing mobile apps for covid-19: Privacy considerations and
  related trade-offs.
\newblock {\em arXiv preprint arXiv:2003.11511}, 2020.

\bibitem{choi2015constructing}
E.~Choi, N.~Du, R.~Chen, L.~Song, and J.~Sun.
\newblock Constructing disease network and temporal progression model via
  context-sensitive hawkes process.
\newblock In {\em 2015 IEEE International Conference on Data Mining}, pages
  721--726. IEEE, 2015.

\bibitem{cox1980point}
D.~R. Cox and V.~Isham.
\newblock {\em Point processes}, volume~12.
\newblock CRC Press, 1980.

\bibitem{cui2020elementary}
L.~Cui, A.~Hawkes, and H.~Yi.
\newblock An elementary derivation of moments of hawkes processes.
\newblock {\em Advances in Applied Probability}, 52(1):102--137, 2020.

\bibitem{daw2018ephemerally}
A.~Daw and J.~Pender.
\newblock An ephemerally self-exciting point process.
\newblock {\em arXiv preprint arXiv:1811.04282}, 2018.

\bibitem{daw2018queues}
A.~Daw and J.~Pender.
\newblock Queues driven by hawkes processes.
\newblock {\em Stochastic Systems}, 8(3):192--229, 2018.

\bibitem{diakonikolas2016learning}
I.~Diakonikolas.
\newblock Learning structured distributions.
\newblock {\em Handbook of Big Data}, 267:10--1201, 2016.

\bibitem{ding2017collecting}
B.~Ding, J.~Kulkarni, and S.~Yekhanin.
\newblock Collecting telemetry data privately.
\newblock {\em Advances in Neural Information Processing Systems}, 30, 2017.

\bibitem{du2015dirichlet}
N.~Du, M.~Farajtabar, A.~Ahmed, A.~J. Smola, and L.~Song.
\newblock Dirichlet-hawkes processes with applications to clustering
  continuous-time document streams.
\newblock In {\em Proceedings of the 21th ACM SIGKDD International Conference
  on Knowledge Discovery and Data Mining}, KDD '15, page 219–228, New York,
  NY, USA, 2015. Association for Computing Machinery.

\bibitem{du2015time}
N.~Du, Y.~Wang, N.~He, J.~Sun, and L.~Song.
\newblock Time-sensitive recommendation from recurrent user activities.
\newblock {\em Advances in neural information processing systems}, 28, 2015.

\bibitem{dwork2006calibrating}
C.~Dwork, F.~McSherry, K.~Nissim, and A.~Smith.
\newblock Calibrating noise to sensitivity in private data analysis.
\newblock In {\em Theory of cryptography conference}, pages 265--284. Springer,
  2006.

\bibitem{dwork2014algorithmic}
C.~Dwork, A.~Roth, et~al.
\newblock The algorithmic foundations of differential privacy.
\newblock {\em Found. Trends Theor. Comput. Sci.}, 9(3-4):211--407, 2014.

\bibitem{dwork2016concentrated}
C.~Dwork and G.~N. Rothblum.
\newblock Concentrated differential privacy.
\newblock {\em arXiv preprint arXiv:1603.01887}, 2016.

\bibitem{etesami2016learning}
J.~Etesami, N.~Kiyavash, K.~Zhang, and K.~Singhal.
\newblock Learning network of multivariate hawkes processes: A time series
  approach.
\newblock {\em arXiv preprint arXiv:1603.04319}, 2016.

\bibitem{farajtabar2015coevolve}
M.~Farajtabar, Y.~Wang, M.~Gomez~Rodriguez, S.~Li, H.~Zha, and L.~Song.
\newblock Coevolve: A joint point process model for information diffusion and
  network co-evolution.
\newblock {\em Advances in Neural Information Processing Systems}, 28, 2015.

\bibitem{feng2019dropout}
W.~Feng, J.~Tang, T.~X. Liu, S.~Zhang, and J.~Guan.
\newblock Understanding dropouts in moocs.
\newblock In {\em Proceedings of the 33rd AAAI Conference on Artificial
  Intelligence}, 2019.

\bibitem{gao2018functional}
X.~Gao and L.~Zhu.
\newblock Functional central limit theorems for stationary hawkes processes and
  application to infinite-server queues.
\newblock {\em Queueing Systems}, 90(1):161--206, 2018.

\bibitem{hall2013random}
R.~Hall, L.~Wasserman, and A.~Rinaldo.
\newblock Random differential privacy.
\newblock {\em Journal of Privacy and Confidentiality}, 4(2), 2013.

\bibitem{hassan2019differential}
M.~U. Hassan, M.~H. Rehmani, and J.~Chen.
\newblock Differential privacy techniques for cyber physical systems: a survey.
\newblock {\em IEEE Communications Surveys \& Tutorials}, 22(1):746--789, 2019.

\bibitem{hawkes1971spectra}
A.~G. Hawkes.
\newblock Spectra of some self-exciting and mutually exciting point processes.
\newblock {\em Biometrika}, 58(1):83--90, 1971.

\bibitem{he2015hawkestopic}
X.~He, T.~Rekatsinas, J.~Foulds, L.~Getoor, and Y.~Liu.
\newblock Hawkestopic: A joint model for network inference and topic modeling
  from text-based cascades.
\newblock In {\em International conference on machine learning}, pages
  871--880. PMLR, 2015.

\bibitem{hillairet2021malliavin}
C.~Hillairet, L.~Huang, M.~Khabou, and A.~R{\'e}veillac.
\newblock The malliavin-stein method for hawkes functionals.
\newblock {\em arXiv preprint arXiv:2104.01583}, 2021.

\bibitem{isham1979self}
V.~Isham and M.~Westcott.
\newblock A self-correcting point process.
\newblock {\em Stochastic processes and their applications}, 8(3):335--347,
  1979.

\bibitem{kairouz2015composition}
P.~Kairouz, S.~Oh, and P.~Viswanath.
\newblock The composition theorem for differential privacy.
\newblock In {\em International conference on machine learning}, pages
  1376--1385. PMLR, 2015.

\bibitem{krishnamurthy2020algebraic}
A.~Krishnamurthy, A.~Mazumdar, A.~McGregor, and S.~Pal.
\newblock Algebraic and analytic approaches for parameter learning in mixture
  models.
\newblock In {\em Algorithmic Learning Theory}, pages 468--489. PMLR, 2020.

\bibitem{laub2015hawkes}
P.~J. Laub, T.~Taimre, and P.~K. Pollett.
\newblock Hawkes processes.
\newblock {\em arXiv preprint arXiv:1507.02822}, 2015.

\bibitem{laxman2008stream}
S.~Laxman, V.~Tankasali, and R.~W. White.
\newblock Stream prediction using a generative model based on frequent episodes
  in event sequences.
\newblock In {\em Proceedings of the 14th ACM SIGKDD international conference
  on Knowledge discovery and data mining}, pages 453--461, 2008.

\bibitem{leskovec2014snap}
J.~Leskovec and A.~Krevl.
\newblock Snap datasets: Stanford large network dataset collection, 2014.

\bibitem{lin2021privacy}
Z.~Lin, V.~Sekar, and G.~Fanti.
\newblock On the privacy properties of gan-generated samples.
\newblock In {\em International Conference on Artificial Intelligence and
  Statistics}, pages 1522--1530. PMLR, 2021.

\bibitem{lv2018correlated}
D.~Lv and S.~Zhu.
\newblock Correlated differential privacy protection for big data.
\newblock In {\em 2018 IEEE 32nd International Conference on Advanced
  Information Networking and Applications (AINA)}, pages 1011--1018. IEEE,
  2018.

\bibitem{mcmahan2017learning}
H.~B. McMahan, D.~Ramage, K.~Talwar, and L.~Zhang.
\newblock Learning differentially private recurrent language models.
\newblock {\em arXiv preprint arXiv:1710.06963}, 2017.

\bibitem{mei2017neural}
H.~Mei and J.~Eisner.
\newblock The neural hawkes process: A neurally self-modulating multivariate
  point process.
\newblock In {\em Proceedings of the 31st International Conference on Neural
  Information Processing Systems}, NIPS'17, page 6757–6767, 2017.

\bibitem{mendes2017privacy}
R.~Mendes and J.~P. Vilela.
\newblock Privacy-preserving data mining: methods, metrics, and applications.
\newblock {\em IEEE Access}, 5:10562--10582, 2017.

\bibitem{mironov2017renyi}
I.~Mironov.
\newblock R{\'e}nyi differential privacy.
\newblock In {\em 2017 IEEE 30th computer security foundations symposium
  (CSF)}, pages 263--275. IEEE, 2017.

\bibitem{moitra2018algorithmic}
A.~Moitra.
\newblock {\em Algorithmic aspects of machine learning}.
\newblock Cambridge University Press, 2018.

\bibitem{ogata1981lewis}
Y.~Ogata.
\newblock On lewis' simulation method for point processes.
\newblock {\em IEEE transactions on information theory}, 27(1):23--31, 1981.

\bibitem{ogata1998space}
Y.~Ogata.
\newblock Space-time point-process models for earthquake occurrences.
\newblock {\em Annals of the Institute of Statistical Mathematics},
  50(2):379--402, 1998.

\bibitem{price2019privacy}
W.~N. Price and I.~G. Cohen.
\newblock Privacy in the age of medical big data.
\newblock {\em Nature medicine}, 25(1):37--43, 2019.

\bibitem{rochLectureNotes}
S.~Roch.
\newblock Modern discrete probability lecture notes: Branching processes.
\newblock 2014.

\bibitem{ross1996stochastic}
S.~M. Ross, J.~J. Kelly, R.~J. Sullivan, W.~J. Perry, D.~Mercer, R.~M. Davis,
  T.~D. Washburn, E.~V. Sager, J.~B. Boyce, and V.~L. Bristow.
\newblock {\em Stochastic processes}, volume~2.
\newblock Wiley New York, 1996.

\bibitem{song2017pufferfish}
S.~Song, Y.~Wang, and K.~Chaudhuri.
\newblock Pufferfish privacy mechanisms for correlated data.
\newblock In {\em Proceedings of the 2017 ACM International Conference on
  Management of Data}, pages 1291--1306, 2017.

\bibitem{tambe2019artificial}
P.~Tambe, P.~Cappelli, and V.~Yakubovich.
\newblock Artificial intelligence in human resources management: Challenges and
  a path forward.
\newblock {\em California Management Review}, 61(4):15--42, 2019.

\bibitem{tang2017privacy}
J.~Tang, A.~Korolova, X.~Bai, X.~Wang, and X.~Wang.
\newblock Privacy loss in apple's implementation of differential privacy on
  macos 10.12.
\newblock {\em arXiv preprint arXiv:1709.02753}, 2017.

\bibitem{tikhomirov1981convergence}
A.~N. Tikhomirov.
\newblock On the convergence rate in the central limit theorem for weakly
  dependent random variables.
\newblock {\em Theory of Probability \& Its Applications}, 25(4):790--809,
  1981.

\bibitem{torrisi2022asymptotic}
G.~L. Torrisi and E.~Leonardi.
\newblock Asymptotic analysis of poisson shot noise processes, and
  applications.
\newblock {\em Stochastic Processes and their Applications}, 144:229--270,
  2022.

\bibitem{xu2016learning}
H.~Xu, M.~Farajtabar, and H.~Zha.
\newblock Learning granger causality for hawkes processes.
\newblock In {\em International conference on machine learning}, pages
  1717--1726. PMLR, 2016.

\bibitem{yang2013mixture}
S.-H. Yang and H.~Zha.
\newblock Mixture of mutually exciting processes for viral diffusion.
\newblock In {\em International Conference on Machine Learning}, pages 1--9.
  PMLR, 2013.

\bibitem{yang2017online}
Y.~Yang, J.~Etesami, N.~He, and N.~Kiyavash.
\newblock Online learning for multivariate hawkes processes.
\newblock In {\em Proceedings of the 31st International Conference on Neural
  Information Processing Systems}, pages 4944--4953, 2017.

\bibitem{zhang2020self}
Q.~Zhang, A.~Lipani, O.~Kirnap, and E.~Yilmaz.
\newblock Self-attentive hawkes process.
\newblock In {\em International conference on machine learning}, pages
  11183--11193. PMLR, 2020.

\bibitem{zhang2020correlated}
T.~Zhang, T.~Zhu, R.~Liu, and W.~Zhou.
\newblock Correlated data in differential privacy: Definition and analysis.
\newblock {\em Concurrency and Computation: Practice and Experience}, page
  e6015, 2020.

\bibitem{zhao2017dependent}
J.~Zhao, J.~Zhang, and H.~V. Poor.
\newblock Dependent differential privacy for correlated data.
\newblock In {\em 2017 IEEE Globecom Workshops (GC Wkshps)}, pages 1--7. IEEE,
  2017.

\bibitem{zhu2017differentially}
T.~Zhu, G.~Li, W.~Zhou, and S.~Y. Philip.
\newblock Differentially private data publishing and analysis: A survey.
\newblock {\em IEEE Transactions on Knowledge and Data Engineering},
  29(8):1619--1638, 2017.

\bibitem{zhu2014correlated}
T.~Zhu, P.~Xiong, G.~Li, and W.~Zhou.
\newblock Correlated differential privacy: Hiding information in non-iid data
  set.
\newblock {\em IEEE Transactions on Information Forensics and Security},
  10(2):229--242, 2014.

\bibitem{zuo2020transformer}
S.~Zuo, H.~Jiang, Z.~Li, T.~Zhao, and H.~Zha.
\newblock Transformer {H}awkes process.
\newblock In {\em Proceedings of the 37th International Conference on Machine
  Learning}, volume 119 of {\em Proceedings of Machine Learning Research},
  pages 11692--11702. PMLR, 13--18 Jul 2020.

\end{thebibliography}
\end{document}